\documentclass[sigconf, nonacm]{acmart}

\usepackage[ruled,linesnumbered]{algorithm2e}

\usepackage{cleveref}
\usepackage{subcaption}
\usepackage{amsmath}
\usepackage{multirow}
\usepackage{multicol}
\usepackage{array}
\usepackage{arydshln}
\usepackage{graphicx}
\usepackage{colortbl}
\usepackage[inkscapelatex=false]{svg}
\usepackage{fancyhdr}

\setcopyright{none}
\acmDOI{}
\acmISBN{}
\acmConference{}{}{}
\acmYear{}

\fancypagestyle{firstpage}{%
  \fancyhf{}
  
  \fancyhead[C]{Preprint for the Accepted Work in The 15th ACM Conference on Data and Application Security and Privacy (CODASPY'25)}
}

\AtBeginDocument{%
  }

\newcommand{\comm}[1]{}
\makeatletter
\newcounter{func}
\newenvironment{func}[1][htb]
  {
   \let\c@algocf\c@func
   \begin{algorithm}[#1]%
  }{\end{algorithm}}

\newcommand{\mymodel}{\ensuremath{\mathsf{NARGIS}}\xspace}
\newcommand{\bful}[1]{\textbf{\underline{#1}}}
\newcommand{\model}{NARGIS}

\newcommand{\module}[1]{\footnotesize\texttt{#1}\normalsize\xspace}

\newcommand{\dottedline}[2]{\cdashline{#1-#2}[.6pt/1pt]}
\newcommand{\fixtableonecol}[1]{
\renewcommand{\arraystretch}{1}
    \fontsize{#1}{#1}\selectfont
    \setlength{\tabcolsep}{0.5pt}
}

\newcommand{\bigline}{\hline\\}

\newcommand{\erf}{\emph{EdgeRand $(\epsilon=6)$}\ }
\newcommand{\lapf}{\emph{LapGraph ($\epsilon=6$)}\ }
\newcommand{\lapd}{\emph{LapGraph ($\epsilon=0.1$)}\ }
\newcommand{\dphfp}{\emph{DP-HFP}}
\newcommand{\dphpp}{\emph{DP-HPP}}
\newcommand{\aster}{($\boldsymbol{*}$)} 
\newcommand{\cross}{(\dag)}
\newcommand{\mymodeldef}{\emph{NARGIS-DefTuned}}
\newcommand{\mymodelpred}{\emph{NARGIS-PredTuned}}
\newtheorem{proposition}{Proposition}

\makeatother
\crefname{algorithm}{algorithm}{algorithms} 
\Crefname{algorithm}{Algorithm}{Algorithms} 
\crefname{func}{function}{functions} 
\Crefname{func}{Function}{Functions} 

\begin{document}

\title[How Feasible is Augmenting Fake Nodes with Learnable Features as a Counter-strategy against Link Stealing Attacks?]{How Feasible is Augmenting Fake Nodes with Learnable Features as a Counter-strategy against Link Stealing Attacks?}

\author{Mir Imtiaz Mostafiz}
\email{mmostafi@purdue.edu}
\affiliation{%
\department{Department of Computer Science}
  \institution{Purdue University}
  \city{West Lafayette}
  \state{Indiana}
  \country{USA}
}

\author{Imtiaz Karim}
\email{karim7@purdue.edu}
\affiliation{%
\department{Department of Computer Science}
  \institution{Purdue University}
  \city{West Lafayette}
  \state{Indiana}
  \country{USA}
}

\author{Elisa Bertino}
\email{bertino@purdue.edu}
\affiliation{%
\department{Department of Computer Science}
  \institution{Purdue University}
  \city{West Lafayette}
  \state{Indiana}
  \country{USA}
}

\renewcommand{\shortauthors}{Mir Imtiaz et al.}

\begin{abstract}
Graph Neural Networks (GNNs) are widely used and deployed for graph-based prediction tasks. However, as good as GNNs are for learning graph data, they also come with the risk of privacy leakage. For instance, an attacker can run carefully crafted queries on the GNNs and, from the responses, can 
infer the existence of an edge between a pair of nodes. This attack, dubbed as a \emph{link-stealing} attack, can jeopardize the user's privacy by leaking potentially sensitive information. To protect against this attack, we propose an approach called $\text{\textbf{\underline{N}}}$ode $\text{\textbf{\underline{A}}}$ugmentation for $\text{\textbf{\underline{R}}}$estricting $\text{\textbf{\underline{G}}}$raphs from $\text{\textbf{\underline{I}}}$nsinuating their $\text{\textbf{\underline{S}}}$tructure (\mymodel{}) and study its feasibility. \mymodel{} is focused on reshaping the graph embedding space so that the posterior from the GNN model will still provide utility for the prediction task but will introduce ambiguity for the link-stealing attackers. To this end, \mymodel{} applies spectral clustering on the given graph to facilitate it being augmented with new nodes- that have learned features instead of fixed ones. It utilizes tri-level optimization for learning parameters for the GNN model, surrogate attacker model, and our defense model (i.e. learnable node features). We extensively evaluate \mymodel{} on three benchmark citation datasets over eight knowledge availability settings for the attackers. We also evaluate the model fidelity and defense performance on influence-based link inference attacks. Through our studies, we have figured out the best feature of \mymodel{}- its superior fidelity-privacy performance trade-off in a significant number of cases. We also have discovered in which cases the model needs to be improved, and proposed ways to integrate different schemes to make the model more robust against link stealing attacks.

\end{abstract}

\maketitle
\thispagestyle{firstpage}  

\section{Introduction}
Graph-based data representations are widely used in online products~\cite{liu2021item}, social networks~\cite{sankar2021graph}, content services ~\cite{wu2022graph}, and web services~\cite{neo4j_aura}. For learning 
graph-based data, different kinds of Graph Neural Network (GNN) architectures, e.g., GCN~\cite{kipf2016semi}, GAT~\cite{veličković2018graph}, 
SAGE \cite{hamilton2017inductive}
have been deployed in these services
to predict user preferences, suggest preferable products, and enhance community structure. These online services gather user data that often contain sensitive information and are often of interest to malicious parties. Hence, these services can sometimes be provided through an API~\cite{facebookgraph} instead of direct access to
the models for security and privacy purposes. Nonetheless, both API and model access are still vulnerable to leaking sensitive information through the GNNs
~\cite{wu2022linkteller,he2021stealing}. When a GNN is trained on a particular graph dataset, malicious parties can run carefully crafted queries on that GNN’s API and, from the responses, can reverse-engineer the graph structures. These forms of attacks are known as Graph Reconstruction Attacks (GIA)~\cite{he2021stealing,zhang2022inference,zhou2023strengthening},  where an adversary has some prior (i.e., node features, dataset name, etc.) or posterior (i.e., APIs providing the probability of what movie a user is likely to watch in a streaming platform etc.) knowledge to recover the relations in the graph. One form of graph reconstruction attack is the \emph{link-stealing attack} (also known as edge inference attack), introduced by He et al.~ \cite{he2021stealing}, where an attacker can run queries on a GNN trained on a graph dataset for node classification tasks, recover the prediction probabilities per node (known and referred to as \emph{posteriors}), and run similarity measure-based unsupervised and supervised learning algorithms to predict the existence of an edge between a pair of nodes. Link-stealing attacks have quantitatively shown the vulnerabilities of graph structures’ privacy while used as a dataset for GNN learning. 
Attackers even often use surrogate models trained on another graph from the same domain to replicate the original GNN functionalities and infer the training graph structure.

\noindent \textbf{Prior Defenses.} 
To protect against link-stealing attacks that leverage
node posteriors, several approaches have been proposed based on Differential Privacy (DP). Zhu et al.~\cite{zhu2023blink} introduced link-local differential privacy, where one introduces 
noise in the localized graph's adjacency matrix of decentralized nodes, for training a GNN in a central server without 
revealing the exact existence of edges. Kolluri et al.~\cite{kolluri2022lpgnet} proposed a new Multi-layer perceptron (MLP) based architecture, 
where the edge information is condensed in clustering-ingrained features and fed into MLPs for node classification task, enabling learning on graphs while obfuscating edge information. While the approach by Zhu et al.~\cite{zhu2023blink} achieves defense guarantees in localized settings, a requirement of this defense is that the nodes perturb their adjacency lists before sending them to the server. Also, in the approach by Kolluri et al.~\cite{kolluri2022lpgnet}, edge information is not explicitly transmitted across the network- rather, they are compressed into a latent space representation before being transmitted. In both cases, the original edge information is modified or lost, and thus the expressiveness of the learned representation of the concerned graphs is reduced. 
Wu et al.~\cite{wu2022linkteller}  discussed different \textit{Differentially Private Graph Convolutional Neural Networks (DP-GCN)} mechanisms (\textit{EdgeRand, LapGraph}) for defending against edge inference attacks. While their approach achieves DP guarantee, both of them have to trade off model utility highly ($\epsilon <=1$) to safeguard privacy, or sacrifice privacy to attain higher utility ($\epsilon >=6)$. Also, \textit{EdgeRand} changes the graph density extensively, often creating \emph{out-of-memory} error for large graphs.

\noindent \textbf{Defense Key Insight.}
A key insight in the  DP-based defenses is the injection of noises to defend the graph structure from being inferred. These noises are added in the adjacency matrix of the graphs~\cite{wu2022linkteller}, which can change the sparsity of the graph to a huge extent (\emph{EdgeRand)} or delete some old edges (\emph{LapGraph}). It is to be noted that noise injection can also be a form of attack. Noise-based perturbation as an attack form is referred to as the \emph{poisoning attacks}. Poisoning attacks modify the graph structures at training time by changing the node features or even inserting fake nodes, leading to compromising the learning ability of the GNNs. The key takeaway is that a defense (noise addition) can also be a form of attack. As DP-based approaches' noise addition has been performed on the edge perspective, it now raises a counter-question, \emph{Can a node-based noise addition (i.e., a poisoning attack through changing node features or adding fake nodes) also be a form of defense against link stealing attacks for the Graphs?}



To answer this question, we analyzed a form of graph modification attacks known as \emph{Graph Injection Attacks} (GIA)~\cite{wang2018attack, zou2021tdgia, sun2020adversarial, wang2020scalable}, where the adversary inserts carefully crafted nodes in a graph (i.e., publishing a fake paper to perturb a citation network, etc.) to mislead the prediction of GNNs through the perturbation of posteriors. The intuition is that if a defense mechanism can take the role of attacker and insert fake nodes into the graph on which the GNN will be trained, it can ``counter-attack’’ the attacker. However, it will also lead to the GNN prediction model being sub-optimized for the original task, as learning the graph with perturbed topology will cause the node posteriors to be perturbed, too. So, the problem can posed as a bi-optimization: perturbing the graph structure and GNN posteriors enough to defend against the attackers but also to provide optimal service to the user within a range. Defenses for ML models through bi-level optimization have been proposed by Wu et al.~ \cite{wuefficient}. Interestingly, it works by learning a noise transition matrix for posterior perturbations. However, this defense is tailored against model-stealing attacks for images and does not cover the graph domain. Also, the optimization is done after the learning to hide the model from being inferred, whereas our target is to optimize during training to hide the training data structure.
Moreover, the user task and the attacker task are the same (image classification) in the stated model. In contrast, in our model, the user focus is to learn to predict node classes, and the attacker's focus is on edge inference.

\noindent \textbf{Motivating Example for Our Approach}
For example, in \Cref{fig:nargmotivate}, three nodes $u_a, u_b , u_c  \in V$ have the posterior output vectors for a two-class classification problem as $\mathbf{a} = [0.3,0.7]$, $\mathbf{b} = [0.1, 0.9]$,  \& $\mathbf{c}=[0.4, 0.6]$, respectively, and only node pair (${u_a}, {u_c}$) have an edge between them in their corresponding graph (\textbf{none of the nodes are drawn in the figure as per scale and orientation}). The Euclidean distance between the posteriors are (pairwise): 
$(d_{ab}, d_{bc}, d_{ca}) = (0.28,0.42,0.14)$.
A link-stealing attacker can set a distance upper-bound threshold of $0.2$, and hypothesize correctly that the node pair ${(u_a, u_c)}$ have an edge between them (as only these nodes are not more distant than the threshold). But if the embeddings are perturbed in a manner such that the posteriors become $\mathbf{a^\prime}= [0.13, 0.87]$, $\mathbf{b^\prime}=[0.35, 0.65]$ \& $\mathbf{c^\prime}=[0.45, 0.55]$, then the classes will remain same still (the second entry being always the highest, inferring all have the label $1$ from between $\{0,1\}$), but the Euclidean distance would become $(d_{a^\prime b^\prime}, d_{b^\prime c^\prime}, d_{c^\prime a^\prime}) = (0.31,0.14,0.45)$
Hence if the link stealer has a threshold of 0.2 again, it will likely conclude that there is an edge between node pair ($u_b, u_c$) and the attack is thus thwarted.

\begin{figure}[!htp]
  \centering
  \includegraphics[scale=0.05]
  {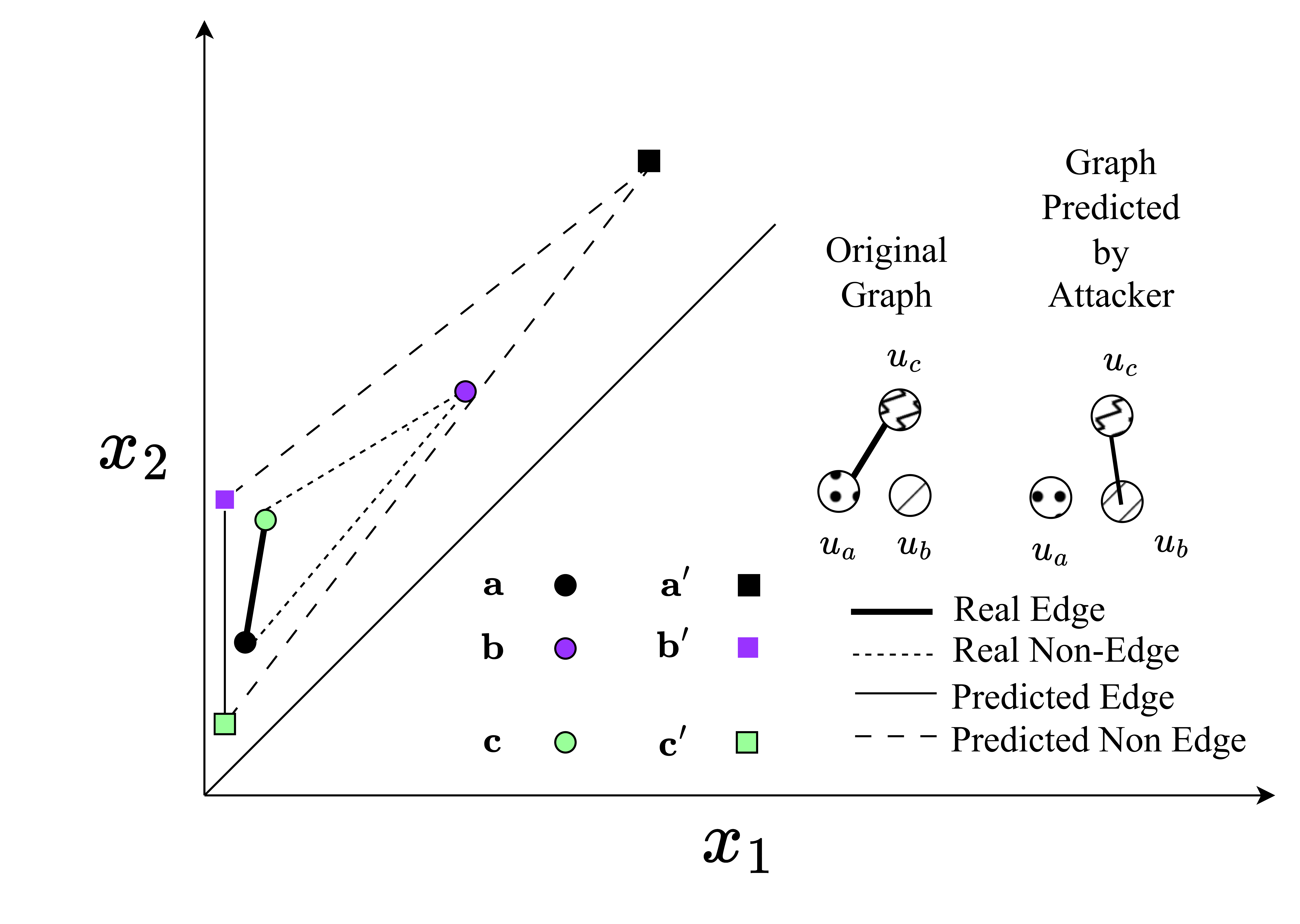}
  \caption{Illustration of Posterior Perturbation in Posterior Simplex Space for Protection against Link Stealing Attacks  
  }\Description{This figure describes an example scenario of how the perturbation of posteriors can be used to defend against Link Stealing Attacks}
  \label{fig:nargmotivate}
\end{figure}
\noindent \textbf{Our Approach.} Our approach focuses on reshaping the graph embedding space in such a way that the posteriors from the GNN model will still provide utility for node prediction but will introduce ambiguity for the link-stealing attackers. As the graph embedding space is constructed through message passing among the nodes, and in a $k-layer$ GNN, this information exchange ranges to the whole $k-$hop neighborhood of a graph; our intuition is that we can augment new nodes with carefully crafted features, whose introduction can perturb the embedding space of a cluster of nodes in their own $k-$neighborhoods. we integrate a form of poison attack in the design of our defense, as our research question demanded. The idea of inserting nodes with crafted features resembles the node injection attack introduced by Zou et al.~\cite{zou2021tdgia}, where the features are learned through optimization instead of being fixed. Besides, to ensure that our new nodes can influence a significant neighborhood in a graph, we apply spectral clustering on the graph first and augment edges from the new node to the center node of the formed clusters. This strategy is different from the one applied by Zou et al.~\cite{zou2021tdgia}, where topologically defected edges are selected under different criteria. Moreover, we introduce a tri-level optimization for learning parameters for our defense model (learnable node features), the model parameters for the target GNN task, and also for a surrogate attacker model which helps to simulate the possible attacks in order to be prepared against. We name our approach as ``\bful{N}ode \bful{A}ugmentation for \bful{R}estricting \bful{G}raphs from \bful{I}nsinuating their \bful{S}tructure'' (\model{}). We evaluate \model{} on three benchmark citation dataset: Cora, Citeseer, Pubmed \cite{yang2016revisiting}.
We also evaluated our approach on LinkTeller \cite{wu2022linkteller}, a state-of-the-art link inference attack. Through our evaluations, we have found some cases where our model and their variants can better the performance of Differential-Privacy (DP) based defenses. Also, we could find out the limitations of our model while defending against some form of attacks.


\noindent \textbf{Contributions.} This paper makes the following technical contributions.

\begin{itemize}
    \item We studied how a form of Graph modification attack can be integrated as a defense for a Graph link-stealing or edge inference attack. To this end, we have introduced \mymodel{}--a node augmentation-based defense for Graph Neural Networks, which can restrict the model from providing posteriors leading to edge inference by attackers, by perturbing the embedding space through \emph{augmenting} nodes and \emph{learning} their features.
    \item We have evaluated our approach on three citation datasets: Cora, Citeseer and Pubmed \cite{yang2016revisiting}, on eight attacking settings stated in \cite{he2021stealing} and on \textit{LinkTeller}~\cite{wu2022linkteller}. We also evaluated node prediction performances on the augmented model. 
    \item We propose that \mymodel{} can be tuned to get near-optimal defensive performances as the DP-based Defenses, with greater fidelity.
    \item We also discussed the cases where the model needs to improve and how different schemes can be used in further investigation to make it perform better.
\end{itemize}

\section{Preliminaries} \label{sec:prel}
 
In this section, we introduce 
preliminary definitions of Graph Learning,
and present the scenarios when 
the privacy of user data, encoded in the graph concerned, are compromised.
We then define the attack our model is trying to defend from.

\subsection{Graph Learning}

\subsubsection{Graph}\label{sec:prel:graph}A graph $G$ is an ordered tuple $(V,E)$ where $V$ is the vertex set and $E$ is the edge set. Edges can be represented by \textit{Adjacency Matrix} $\mathbf{A}$; for a graph with $n$ nodes, i.e., $\lvert V \rvert = n$, $\mathbf{A} \in \{0,1\}^{n\times n}$. Formally, for $\mathbf{A}$, $\mathbf{A}_{ij} = [(v_i, v_j) \in E]$. Node attributes are represented as a matrix $\mathbf{X} \in \mathbb{R}^{n \times d}$ with feature dimension $d \in \mathbb{N}$. The label set $\mathbf{Y} \in \{0,1,\ldots, c-1\}^{n \times 1}$ is a column vector denoting the label of each node (from candidate labels $C = \{0, 1, \ldots, c-1\}$). Thus we can represent a Graph $G$ as an ordered 6-tuple $(V, \textbf{A}, \textbf{X}, \textbf{Y}, n, c)$.
\subsubsection{Node Embedding} An \textit{embedding} function is a mapping from a high dimensional space to a low-dimensional one that can preserve particular properties of the domain space. For a graph $G$, a \textit{Node Embedding} is a function $f: G \rightarrow \mathcal{H} \in \mathbb{R}^{n \times d_e}$, where $d_e << d$ is the \textit{embedding dimension} and $\mathcal{H}$ is the \textit{Embedding space}. For a graph, the most known practice is to find an embedding that preserves the neighborhood similarities or distances among the nodes. 



\subsubsection{Graph Neural Network (GNN) Learning} A \textit{Graph Neural Network (GNN)} is a neural network, which takes a graph $G$ 
as input and generates a node embedding $f:G \rightarrow \mathcal{H}$ for each of its nodes, from the edge relationships. The most popular mechanism for learning these embeddings is called \textit{Message-Passing}. Under this mechanism, in each layer of the \textbf{GNN}, for every single node, two operations are performed: (1) \textbf{Aggregation}: the concerned node's neighbors' node embeddings from the previous layer are collected and combined; (2) \textbf{Update}: the combined neighborhood node embeddings, along with the concerned node's previous layer embedding, are used to calculate the node's embedding for this current layer. In this way, nodes related to each other through multiple hops of the graph can influence each other in the latent embedding space. Formally, for a $L$- layered \textbf{GNN} ($L \in \mathbb{N}$), 
the calculation of
the hidden representation of a node $v \in V$ in the $l \in \{1, 2, \ldots, L-1\}$, denoted as $h_v^{(l)}$, (where $\mathcal{N}_v$ means the neighborhood of $v$ in $G$, i.e. $\mathcal{N}_v = \{u \in V| (v,u) \in E \}$), is executed as follows: 
\begin{align}
    v_{\mathcal{N}}^{(l-1)}  = \text{\module{AGGREGATE}}(
    \mathop{\forall}
    \limits_{\substack{u \in \\ \mathcal{N}_v}}
    h_u^{(l-1)}), h_v^{(l)} = \text{\module{UPDATE}}(h_v^{(l-1)},  v_{\mathcal{N}}^{(l-1)})
\end{align}

Where $\text{\module{UPDATE}}, \text{\module{AGGREGATE}}$ are 
designated differentiable functions. As an initial value, $h_v^{(0)} = \mathbf{X}_v$ is chosen. Notable \textit{Message-passing} \textbf{GNN}s include Graph Convolutional Network (GCN) \cite{kipf2016semi}, Graph Attention Network (GAT) \cite{veličković2018graph}, Graph Isomorphism Network (GIN) \cite{xu2018powerful}, Graph Sample and Aggregate (GraphSAGE) \cite{hamilton2017inductive}. 

In this work, we focus on two of the most important graph learning tasks: (1) \emph{Node Classification}: a node's attached class label is predicted by training on a small subset of labeled nodes, and in the process, node embeddings are learned; and, (2) \emph{Link Prediction}: we learn the embeddings of a node, and for a pair of nodes, we train a neural network or any unsupervised method to find out whether an edge exists between them. The defense side's model is based on node classification, whereas the attacker focuses on the link prediction task.

\subsection{Attacker Model}
We now describe the threat model of our work in terms of the attack model's goal, 
knowledge, 
and capabilities. 
The model in consideration is a GNN, representing a given graph (e.g., a social network or online shopping recommender), trained for the node classification task. At the last layer of the GNN, for each node in the graph, we get a probability distribution (posterior) of the candidate set of node classes. As an example, for a citation network, the nodes representing the works maybe classified as either one from the set \textit{\{Representation Learning, Reinforcement Learning, Meta-Learning\}}, and hence the posterior probability distribution will be a three-dimensional vector representing a simplex.
\subsubsection{Attacker Goal} The attacker will try to infer the graph structure, i.e., finding the existence of an edge between two given nodes. If the attacker has knowledge about the nodes and for each pair of nodes can infer whether there exists an edge between them, then they can eventually reconstruct
the whole graph.
For social or product recommendation networks, these edges can be 
of high interest for malicious parties. The attack goal is neither 
under poison settings (e.g., corrupting the graph), nor under evasion settings (e.g., evading the defense). 
Rather, it is a reconstruction setting- aiming to find out the graph's structure.
\subsubsection{Attacker Knowledge}\label{sec:att_know} The attacker's knowledge can be discussed along two different aspects: the GNN model and the training data/Graph. 

The attacker is completely in the dark about the GNN model parameter and hyperparameters. The attacker can access to the posterior distribution of node classes from the GNN model. 

In case of the data or the graph needed for training, the attacker can face multiple scenarios as stated in~\cite{he2021stealing}, depending on the availability of notable Graph constituents. Namely, they are: (1) \textbf{Target Dataset’s Nodes’ Attributes $\mathcal{F}$}: \emph{The attacker sometimes can have the knowledge about the attributes of the nodes $\mathcal{F}$ and labels of a small subset of data used to train the GNN model.} (2) \textbf{target Dataset’s Partial Graph $\mathcal{A}$}: \emph{A subset of graph edges can also be provided to the attacker, to be used as ground truth edges to train the link inference model} and (3) \textbf{Shadow Dataset $\mathcal{D'}$}: \emph{a shadow dataset  can be provided to facilitate transfer learning to mimic the original graph. The attacker trains a shadow target model from a graph with own nodes and edges, from same or different domain. }
Depending on the presence/absence of these three auxiliaries, $2^3=8$ attacking scenarios has been considered in \cite{he2021stealing}. The background knowledge is represented as an ordered tuple $\mathcal{K} = (\mathcal{F}, \mathcal{A}, \mathcal{D'})$, which can have eight different values ranging from $(\times,\times,\times)$ (i.e., all absent) to $(\mathcal{F},\mathcal{A},\mathcal{D'})$ (i.e., all present). All the possible settings are described in \Cref{tab:attack_config_intro}. Also, in our experimental study (\Cref{sec:exp}), we have described all the attack scenarios  before discussing our model's performance under these settings.




\begin{table}[!htp]
\caption{Attacker Knowledge Configurations for Link Stealing Attacks, as described in \cite{he2021stealing}}
    \label{tab:attack_config_intro}
\centering
\fixtableonecol{6}

\begin{tabular}{cc|c|c|c|cc|c|c|c}
\begin{tabular}[c]{@{}c@{}}Attack \\ No.\end{tabular} & \begin{tabular}[c]{@{}c@{}}Node \\ Attribute,\\ $\mathcal{F}$\end{tabular} & \begin{tabular}[c]{@{}c@{}}Partial \\ Graph,\\ $\mathcal{A}$\end{tabular} & \begin{tabular}[c]{@{}c@{}}Shadow\\ Dataset,\\ $\mathcal{D'}$\end{tabular} & \begin{tabular}[c]{@{}c@{}}Background\\ Knowledge,\\ $\mathcal{K}$\end{tabular} & \begin{tabular}[c]{@{}c@{}}Attack \\ No.\end{tabular} & \begin{tabular}[c]{@{}c@{}}Node \\ Attribute,\\ $\mathcal{F}$\end{tabular} & \begin{tabular}[c]{@{}c@{}}Partial \\ Graph,\\ $\mathcal{A}$\end{tabular} & \begin{tabular}[c]{@{}c@{}}Shadow\\ Dataset,\\ $\mathcal{D'}$\end{tabular} & \begin{tabular}[c]{@{}c@{}}Background\\ Knowledge,\\ $\mathcal{K}$\end{tabular} \\\hline\\
Attack-0                                              & $\times$                                                                   & $\times$                                                                  & $\times$                                                                   & $(\times, \times, \times)$                                                      & Attack-4                                              & $\times$                                                                   & $\checkmark$                                                              & $\checkmark$                                                               & $(\times, \mathcal{A}, \mathcal{D'})$                                           \\
Attack-1                                              & $\times$                                                                   & $\times$                                                                  & $\checkmark$                                                               & $(\times, \times, \mathcal{D'})$                                                & Attack-5                                              & $\checkmark$                                                               & $\times$                                                                  & $\checkmark$                                                               & $(\mathcal{F}, \times, \mathcal{D'})$                                           \\
Attack-2                                              & $\checkmark$                                                               & $\times$                                                                  & $\times$                                                                   & $(\mathcal{F}, \times, \times)$                                                 & Attack-6                                              & $\checkmark$                                                               & $\checkmark$                                                              & $\times$                                                                   & $(\mathcal{F}, \mathcal{A}, \times)$                                            \\
Attack-3                                              & $\times$                                                                   & $\checkmark$                                                              & $\times$                                                                   & $(\times, \mathcal{A}, \times)$                                                 & Attack-7                                              & $\checkmark$                                                               & $\checkmark$                                                              & $\checkmark$                                                               & $(\mathcal{F}, \mathcal{A}, \mathcal{D'})$\\\hline                                     
\end{tabular}
\end{table}

\subsubsection{Attacker Capabilities.} 

 The attacker can only run queries on the GNN model to find out the posteriors of the predicted node classes. Augmented with the background knowledge, as described in \Cref{sec:att_know}, the attacker will try to infer the graph edges. 
\subsubsection{Attack Type} The attacker will try a similarity-based attack by leveraging 
the heuristic that nodes with similar posterior class distribution are expected to have edges between them. This type of attack is called \textbf{link stealing attack} as per the literature~
\cite{he2021stealing} and 
\cite{zhang2023demystifying}. The attacker uses different similarity or distance metrics(e.g., cosine similarity, correlation coefficient) to find similar node-pairs.
Formally, given a black box \textbf{GNN} model $\mathcal{G}$, the training (target) graph nodes $V$, the nodes' posteriors from $\mathcal{G}$ ,the background knowledge $\mathcal{K}$, and two nodes $u, v 
\in V$, the link stealing attack allows one to determine whether there exists an edge in the target graph between $u$ and $v$. 
\section{Challenges and Solutions}\label{sec:chalover}
The design of our approach, based on the idea of
\emph{adding new, fake nodes in the graph with crafted features to perturb the embedding space enough to thwart the link-stealing attacker but be sufficiently accurate enough for the model user},
requires addressing several challenges:
    \textbf{(C1)} \textit{Where should we augment the nodes?} The $k-$layer embeddings of a node depend on its $k-$hop neighbors. Therefore, any change in the embeddings of its neighbor nodes also changes the embedding of the node. Thus, 
    we need to connect new nodes to such nodes in the original graph that are already connected with a significant number of nodes through a minimal number of hops. To address this challenge, we apply \textit{spectral clustering}~\cite{von2007tutorial} to partition the graph into different clusters and find the cluster centers (the node with the least intra-cluster average distance). Then, we add edges from the new nodes to the cluster center. As a cluster center itself is connected to many nodes with one or two hops, connecting an edge from a new node to this node will help to perturb the embeddings of the nodes in its proximity efficiently.
    \textbf{(C2)} \textit{What should be an optimal bound for the number of new nodes?}
    We need to find the optimal number of new nodes so that each of them, influencing their cluster, can collectively perturb the whole graph's embedding and posterior space. We hypothesize that this optimal number depends on the \emph{graph's density, which is roughly inversely proportional} to it. As in a dense graph, more nodes can be connected in one or two hops in message-passing scheme for GNNs due to higher number of edges, less cluster-center nodes are needed to perturb the embedding and posterior spaces of neighborhood nodes. We state and prove the theoretical result below:

    \begin{proposition}
Let the Spectral Clustering Algorithm \cite{von2007tutorial} be applied on an unweighted graph $G=(V,E)$ in such a way that (i) every cluster is equally (approximately) sized in terms of the number of nodes, (ii) each node in a cluster is within the $L-$neighborhood of other nodes in the same cluster for a fixed $L \in \mathbb{N}$, and (iii) the highest degree possible within a cluster for a node is $K$ for a fixed $K \in \mathbb{N}$, then to ensure the mentioned constraints, the number of clusters needed to form, $c$, is inversely proportional to the graph density $\sigma$, i.e. 
\begin{displaymath}
    c \propto \frac{1}{\sigma}
\end{displaymath}
\end{proposition}

\begin{proof}
The number of nodes and edges in $G$ are $n=\lvert V \rvert, e = \lvert E \rvert$, respectively. Let the lowest and highest possible number of nodes in a cluster be ${n_c}^{min}$ \& ${n_c}^{max}$, respectively. Then $n$ can be bounded as,
\begin{equation}\label{eq:node}
    {n_c}^{min} * c \leq n \leq {n_c}^{max} * c
\end{equation}
Let the expected number of edges within a cluster to form a minimally connected component following the $L$-neighborhood and $K-$maximum degree assumption is $e_c$. If each cluster is considered as a super-node, and these super-nodes are connected to form minimum-spanning tree of clusters- then for a graph, there are three kinds of edges formed: intra-cluster minimally constrained connected component edges, inter-cluster minimum spanning tree edges, and the extra inter-cluster edges (not necessary to form the tree of clusters). As there are $c$ clusters, there will be $c-1$ inter-cluster minimum-spanning tree edges. We can bound the number of edges $e$ as, 
\begin{equation}\label{eq:edge}
    c * e_c + (c -1) \leq e
\end{equation}
To deduce the lower and upper bound for $e_c$, we consider two cases of the edge structure of the clusters. As a lower bound case, we consider the minimum spanning tree of the nodes in a cluster; and as an upper bound case, we consider the complete $K-ary$ tree of the nodes. For the second case, let the highest degree (intra-cluster) node be the root. Then there will be at most $K$ nodes connected with it in the first level. In the next level, every node will have $K-1$ children (as the highest intra-cluster degree is $K$ and they already have a parent), so there will be $K*(K-1)$ edges. As all the nodes are within an $L-$neighborhood, the tree will have $L$ levels, and the total number of edges will be $K + K*(K-1) + K*(K-1)^2 + \ldots + K*(K-1)^{(L-1)}= \frac{K * \{(K-1)^L-1\}}{K-2} = {f_{max}(K,L)}$. For the first case, we cannot get an explicit formula for the number of edges in terms of $K,L$ as it depends on the graph structure. Nevertheless, we will denote it as $f_{min}(K,L,G)$. As both case denotes a tree, the number of nodes will be 1 more than the number of edges- so ${n_c}^{max} = 1 + {e_c}^{max} = f_{max}(K,L)$,  ${n_c}^{min} = 1 + {e_c}^{min} = 1 + f_{min}(K,L,G)$, and  The density $\sigma$ of a Graph is, 

\begin{equation}\label{eq:density}
    \sigma = \frac{e}{n*(n-1)} \approx \frac{e}{n^2}.
\end{equation}

If every cluster is formed as the first case (minimum spanning tree), then from \Cref{eq:node,eq:edge,eq:density}, the lowest bounded density will be $\sigma_{min} \approx \frac{{e}^{min}}{({{n}^{min}})^2} = \frac{c * {e_c}^{min} + (c -1)}{({n_c}^{min} * c)^2} = \frac{c * ({e_c}^{min} + 1) -1}{({n_c}^{min} * c)^2}=\frac{c * {n_c}^{min}  -1}{({n_c}^{min} * c)^2} \approx \frac{c * {n_c}^{min} }{({n_c}^{min} * c)^2}= \frac{1}{{n_c}^{min} * c} = \frac{1}{c * f_{min}(K,L,G)}$, and the highest bound will be, $\sigma_{max} \approx \frac{{e}^{max}}{({{n}^{max}})^2} = \frac{c * {e_c}^{max} + (c -1)}{({n_c}^{max} * c)^2} = \frac{c * ({e_c}^{max} + 1) -1}{({n_c}^{max} * c)^2}=\frac{c * {n_c}^{max}  -1}{({n_c}^{max} * c)^2} \approx \frac{c * {n_c}^{max} }{({n_c}^{max} * c)^2}= \frac{1}{{n_c}^{max} * c} = \frac{1}{c * f_{max}(K,L)}$ 
\end{proof}

As from the assumptions, the Graph is fixed and so are the values of $K,L$. So, $f_{min}(K,L, G), f_{max}(K,L)$ are both constants. Hence in both cases, $c \propto {\frac{1}{\sigma}}$
    
    Therefore, if two graphs $G_1, G_2$ have densities $\delta_1, \delta_2$, respectively, and defending $G_1$ is achieved optimally with $N$ new nodes, then $G_2$ should be augmented with approximately $\lfloor{\frac{N\delta_1}{\delta2}}\rfloor$ nodes.
    \textbf{(C3)} \textit{How can the competing objectives of optimizing model utility and defending link-stealing attacks be achieved?} The target GNN model's posteriors are optimized for the node prediction task. However, as good as these posteriors are, homophily (edge endpoints having the same labels) is usually found in graph datasets, which exposes their vulnerability in similarity-based attacks. Thus, the posteriors individually have to be optimized but jointly have to be misleading. To address the competing objectives, we introduce a \emph{tri-level optimization} of augmented node features, target GNN parameters, and surrogate attacker model parameters. We optimize augmentation (and GNN parameters), along with surrogate attacker parameters, interchangeably so that one model's gradient feedback updates the other one.   

To summarize, our approach needs functionalities for graph clustering and a multi-level optimization loop to maximize the target model's utility and defensive strength.

\section{Modulewise Detailed Operation of \mymodel{}}\label{sec:methodology}


\begin{figure*}[ht]
\centering
\begin{subfigure}[b]{\textwidth}
\begin{minipage}{\textwidth}
  \includegraphics[width=\linewidth]{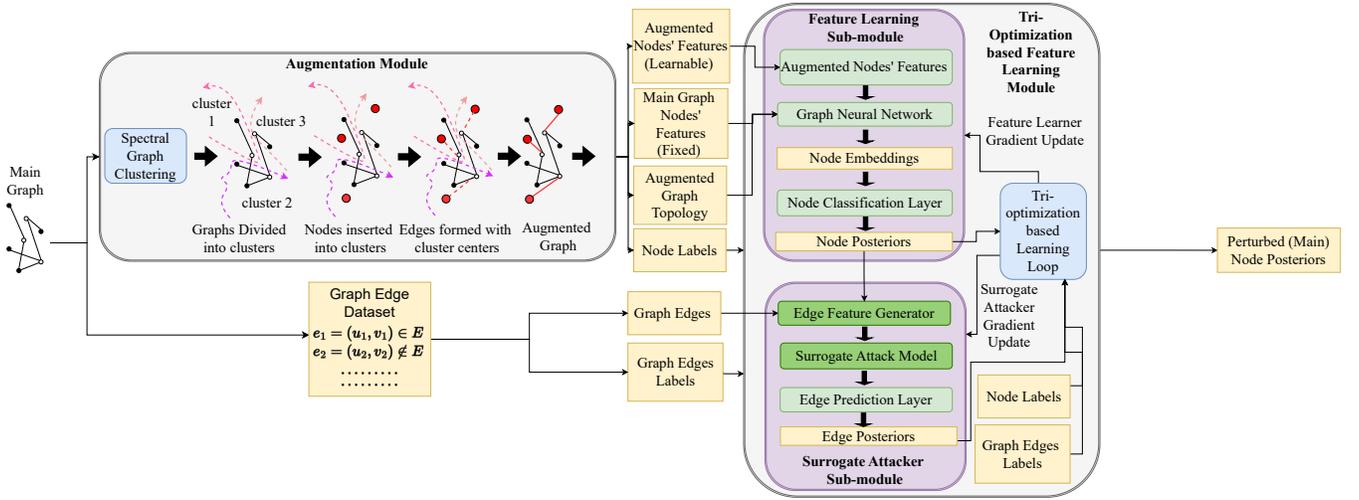}
  
  \caption{Detailed System Design of NARGIS}
  \label{fig:nargis_overview}
\end{minipage}%
\end{subfigure}
\begin{subfigure}[b]{\textwidth}
\begin{minipage}{\textwidth}
  \includegraphics[width=1.1\linewidth]{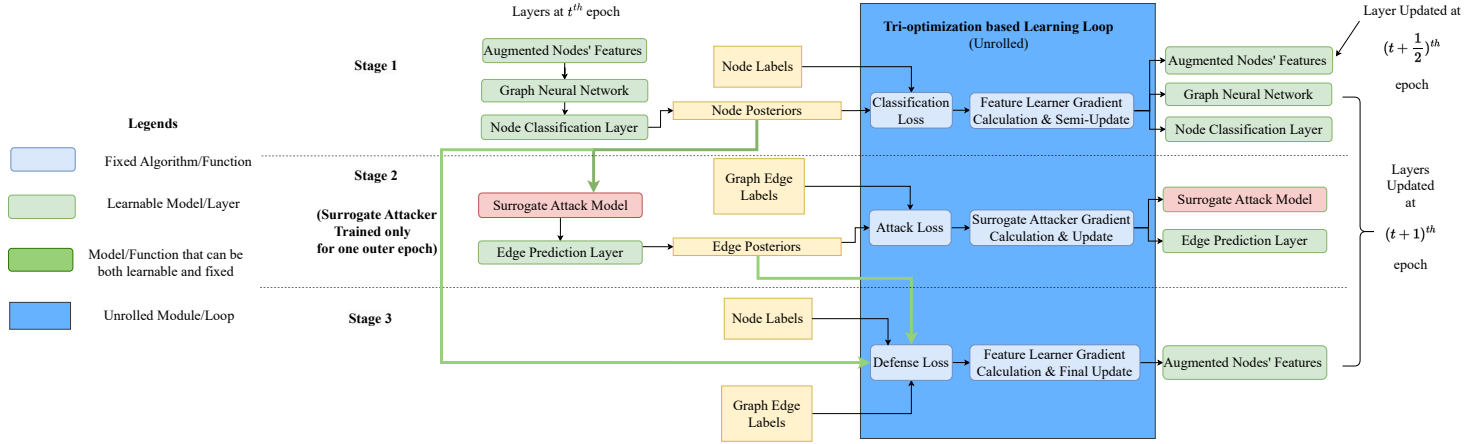}
  
  
  \caption{Tri-Optimization based Feature Learning Model Elaborated (\textbf{GNN-related inputs to Augmented Nodes' Feature, as shown in \Cref{fig:nargis_overview}, are not shown here for the sake of simplicity.})}
  \label{fig:nargis_triopt}
\end{minipage}
\end{subfigure}

\caption{System Overview of \mymodel{}}
\Description{This image describes the system overview of \mymodel{}}
\label{fig:nargis_whole}
\end{figure*}

In this section, we discuss each module separately and together to form the workflow of \mymodel{}. 

\subsection{{Augmentation Module}}
The \module{Augmentation Module} (see \Cref{fig:nargis_overview}), takes the input graph, applies spectral clustering on it, and returns the augmented graph with new nodes and edges. In this module, the new nodes' features are still set at $0$ as they have not been learned yet.


In \mymodel{}, we augment the graph $G$ with new nodes and create some edges from the newly added nodes to the original nodes. Let $n_{new}$ be the number of newly added nodes in the graph $G$ and $V_{new}$ be the set of newly added nodes, then $\lvert V_{new}\rvert = n_{new}$. 
Let the node features and labels of newly augmented nodes be, respectively, $\textbf{X}_{new} \in \mathbb{R}^{n_{new} \times d}$ and $\textbf{Y}_{new} \in  \{0,1,\ldots, c-1\}^{n_{new \times 1}}$. Also, let $\textbf{A}_{new} \in \{0,1\}^{n \times n_{new}}$ be the incidence matrix between the old nodes and newly added ones, with, 

{\small
\begin{align}
\label{eqn:anew}
{[\textbf{A}_{new}]}_{(u,v)} = [\exists \text{ an edge}(u,v) \text{ for } u \in V \text{ and } v \in V_{new}]
\end{align}
}
Let $G_{aug} = (V_{aug}, \textbf{X}_{aug}, \textbf{Y}_{aug}, \textbf{A}_{aug}, n_{aug}, c) $ be the augmented graph, where 

{\small
\begin{align}n_{aug} = n+n_{new},\text{ } V_{aug} = V \cup V_{new},\text{ }\lvert V_{aug} \rvert = n_{aug}\label{eqn:naugvaug}
\end{align} 
\label{eq:anew}
}
%
%
{\small
\begin{align}\textbf{X}_{aug} = \begin{bmatrix}\textbf{X} \\\textbf{X}_{new}\end{bmatrix}, \text{ }\textbf{A}_{aug} = \begin{bmatrix}
\textbf{A} & \textbf{A}_{new} \\{\textbf{A}_{new}^T} & 0\end{bmatrix}, \text{ }\textbf{Y}_{aug} = \begin{bmatrix}
\textbf{Y}_{new}\end{bmatrix}\label{eqn:xaugaugyaug}
\end{align} 
}




Note that we train the \textbf{GNNs} in semi-supervised style 
~\cite{kipf2016semi}, where only a portion of nodes are labeled, and among them, train, validation and test splits are made. Thus, for $\textbf{Y}_{new}$, we randomly set the labels, as they will not be part of any labeled splits. For edge augmentation, we use the Unnormalized Laplacian Based Algorithm for Spectral Clustering~\cite{von2007tutorial}, which finds the graph clusters and their center nodes. We connect each new node with separate cluster centers to form the new edges.
\begin{func}[]
\renewcommand{\arraystretch}{1}
 \fontsize{8}{8}\selectfont
\SetKwInOut{Input}{Input}
\SetKwInOut{Output}{Output}
\SetKwInOut{Parameter}{Parameter}

\SetKwComment{Comment}{/* }{ */}
\RestyleAlgo{ruled}

\Input{Graph $G = (V, \textbf{X}, \textbf{Y}, \textbf{A}, n, c);$\\}
\Parameter{Number of new nodes to augment, $n_{new}$;\\}
\Output{Augmented Graph $G_{aug} = (V_{aug}, \textbf{X}_{aug}, \textbf{Y}_{aug}, \textbf{A}_{aug}, n_{aug}, c)$ with no features learned yet;}

Cluster Center Nodes $[o_1, o_2, \ldots, o_{n_{new}}] \gets$ Spectral Clustering on $G$ as per the algorithm in~\cite{von2007tutorial}\; 
$V_{new} \gets \{v_1, v_2, \ldots v_{n_{new}}\},$ New nodes with zeroes set as feature\;
Form edge set $E_{new} \gets \{(o_1, v_1), (o_2,v_2), \ldots, (o_{n_{new}}, v_{n_{new}})\}$\;
$X_{new} \gets \{\mathbf{0}, \mathbf{0}, \mathbf{0}, \ldots, \mathbf{0}\}, $ $n_{new}$ zero feature vectors\;
$\mathbf{A}_{new} \gets $ form adjacency matrix as per \Cref{eqn:anew} using $V, V_{new} \text{ \& } E_{new}$\;
$\mathbf{Y}_{new} \gets n_{new}$ randomly assigned labels from ${0,1,\ldots, c-1}$\;
Form augmented graph $G_{aug} = (V_{aug}, \textbf{X}_{aug}, \textbf{Y}_{aug}, \textbf{A}_{aug}, n_{aug}, c)$ using \Cref{eqn:naugvaug,eqn:xaugaugyaug} \; 
\Return{$G_{aug}$}\;
\caption{getAugmentedGraph$(G)$}\label[func]{func:getAugGraph}
\end{func}
~\Cref{func:getAugGraph}  represents the functionality of the \module{Augmentation Module}.

\subsection{Tri-Optimization based Feature Learning Module} 
\subsubsection{Module Workflow Summary}\label{sec:triopt-summary}
The \module{Tri-Optimization based Feature Learning Module} is the most crucial module.
Its goal is to learn the augmented nodes' features so that their existence in a GNN learning step will perturb the embedding and posterior space of the original graph's nodes. This module takes both the unlearned augmented graph (from the \module{Augmentation Module}) and main graph (training) edges and their labels (existent/non-existent) as inputs. The module consists of our target GNN model (\module{Feature Learning Sub-Module}) and a surrogate attacker model (\module{Surrogate Attacker Sub Module})- connected through a tri-optimization-based learning loop (see  \Cref{fig:nargis_overview}). The detailed process is shown in \Cref{fig:nargis_triopt}. The target GNN model is trained for the node prediction task in the augmented graph in \textbf{Stage 1}. Consequently, the learned posteriors are sent to the surrogate attacker model, which combines the surrogate query edges with them to form an edge feature and then the attacker model is trained to infer the existence of the edge (\textbf{Stage 2}). Gradient feedback from this model is sent again to the \module{Feature-Learning sub-module} so that the augmented nodes' features are updated through gradient descent, whereas the GNN parameters are kept frozen (\textbf{Stage 3}). Finally, the module returns the learned features of the augmented nodes. It is to be noted that the \module{Surrogate Attacker Sub-Module} is trained only once, unlike the \module{Feature Learning Sub-Module} which is trained several times. Our intuition behind this learning strategy is that the attacker's model is targeted to learn the implicit bias that similar nodes from the edges mainly- and in the first training of the \module{Feature Learning Sub-Module}, similar nodes' posteriors are trained to be similar. Hence, the surrogate attacker takes those posteriors to integrate the inductive bias regarding the homophily. But in later part of the trainings, when the \module{Feature Learning Sub-module} also gets gradient feedback to perturb the embedding and posterior space for the defense against attacks, training the surrogate attackers on those perturbed posteriors again will lead the attacker to be adaptive, which is not under our threat model settings. Making the surrogate attacker adaptive will unnecessarily weaken the defense's power, as no matter what gradient update is done in the \module{Feature Learning Sub-module}, the \module{Surrogate Attacker Sub-module} will always have an answer for that and the defensive objective will not be achieved. This procedure is elaborated in \Cref{func:runtriopt}. Finally, this module returns the perturbed posteriors.

\subsubsection{Module Workflow Details}\label{sec:triopt-details}
\textbf{In the subsequent discussions, we use $\mathcal{D}_{Q}$ to denote the training edges, both positive/existent and sampled negative/non-existent edges from the main graph.}

{\small
\begin{align}
 \mathcal{D}_Q = \{(u, v, y)| u \in V, v \in V, y = A[u][v]\}    
\end{align}\label{eq:dq} 
}
\paragraph{Target GNN Model}\label{sec:targetgnn}
Let $f(G; \theta_f): G \rightarrow [0,1]^{n \times c}$ be the \textit{target model}, which takes a graph $G$ as input and outputs the probability distribution of the nodes' labels (i.e., posteriors). This is the model \mymodel{} defends. \emph{The set of learnable model parameters are denoted as $\theta_f$} (they include both the embedding and classification layer parameters).

\paragraph{Augmented Target GNN Model} 
Let $f_\tau(G_{aug}; {\theta_{f_{\tau}}}, \theta_\tau): G_{aug} \rightarrow [0,1]^{n_{aug} \times c}$ be the \textit{augmented target GNN model}, which takes the augmented graph $G_{aug}$ as input and outputs the perturbed probability distribution (posteriors) of the nodes' labels. This model is used exclusively to learn the augmented node features. \emph{The set of learnable target GNN model parameters are denoted as ${\theta_{f_\tau}}$}, and \emph{the associated learnable augmented node features are: $\theta_\tau$}.


\paragraph{Edge Feature Generator}\label{sec:edgefeature} 
This procedure takes the node posteriors from the target GNN model and surrogate query edges to form the edge features for the \textit{surrogate attacker model}. 
Formally, let $\mathbf{\chi} = f_\tau(G_{aug}; {\theta_f}, \theta_\tau)$ be the node posteriors from \textit{the augmented target model}. Then, the \textit{Edge Feature Generator} can be denoted as a function $k(\mathbf{\chi}, \mathcal{D}_{Q}) \rightarrow \mathbb{R}^{\lvert \mathcal{D}_{Q} \rvert \times d_a}$, where $d_a$ be \textit{the surrogate attack model} input dimension.

\paragraph{Surrogate Attack Model}\label{sec:surrattackmodel}
We refer to $g(u,v, k(\mathbf{\chi}, \mathcal{D}_{Q}); \theta_g)\rightarrow [0,1]$ 
as \textit{Surrogate Attack Model} for link prediction between two nodes $u,v$ in a (non-)existing edge in $\mathcal{D}_{Q}$. \emph{$\theta_g$ denotes the set of learnable parameters of the model}. It takes the surrogate dataset and the formed edge features as input for learning.


\paragraph{Module Optimization Objective} 
The goal is to learn the features of the augmented nodes. As the aim of any Learning-based model's defense is two-fold (providing utility while defending from attacks), to assist the learning of features, we learn the target model for the node classification task (to help the learnable nodes improve the utility) as well as a surrogate attacker model (to help the learnable nodes be robust against attacks).



Formally, let the node classification loss of the aforementioned \textit{Augmented Target GNN model} $f_{\tau}(.)$ be  $\mathcal{L}_{class}(G_{aug};\theta_{f_\tau}, \theta_{\tau})$. We also introduce an edge classification loss of the \textit{Surrogate Attack Model} $g(.)$ as  $\mathcal{L}_{attack}(\mathcal{D}_{Q}, \chi;\theta_{g})$. 
Lastly, based on all those learnable parameters, our loss of interest is the loss dedicated to the node augmentation-based defense, $\mathcal{L}_{defense}(\mathcal{D}_{Q}, G_{aug};\theta_g,\theta_{f_\tau}, \theta_{\tau})$. The detailed definition of the loss functions are given in what follows.

\paragraph{Classification loss}
For \textit{augmented target GNN model}'s classification loss $\mathcal{L}_{class}$, we use the negative log-likelihood loss: 

{\small
\begin{align}
\mathcal{L}_{class}(G_{aug}; \theta_{f_\tau}, \theta_{\tau}):= \mathbb{E}_{u \sim V}[-log([f_\tau(G_{aug}; \theta_{f_\tau, \tau})]_{(u,\mathbf{Y}_{u})})]
\label{eqn:classloss}
\end{align}
}

Here, we obtain the posteriors from the \textit{augmented target GNN model}, and for each node-label pair $(u,\mathbf{Y}_{u}) \in V_{aug} \times \mathbf{Y}_{aug}$ of the graph $G_{aug} = (V_{aug}, \textbf{X}_{aug}, \textbf{Y}_{aug}, \textbf{A}_{aug}, n_{aug}, c) $, we take the probability value of the true node label, and take the expected value of negative log-likelihood over all the nodes in the augmented graph that belongs actually to the original graph (as $V \subset V_{aug}$).

\paragraph{Surrogate Attacker Loss}
For the \textit{surrogate attack model}'s link prediction loss, we use the binary cross-entropy loss: 

{\small
\begin{align}\begin{split}
    \mathcal{L}_{attack}(\mathcal{D}_{Q}, \mathbf{\chi}; \theta_{g})=\mathcal{L}_{attack}(\mathcal{D}_{Q}, G_{aug}; \theta_{g},\theta_{f_\tau}, \theta_{\tau})\\:=\mathbb{E}_{(u,v,y) \sim \mathcal{D}_Q}[\mathcal{L}_{BCE}(g(u,v, k(\mathbf{\chi}, \mathcal{D}_{Q}); \theta_g), y)]
    \label{eqn:lossatk}
\end{split}
\end{align}
}

where $g(u,v, k(\mathbf{\chi}, \mathcal{D}_{Q}); \theta_g)$ denotes the probability of an edge between nodes $u,v \in V$ from the surrogate model $g$ by the attacker 
and $y$ is a binary variable indicating whether there is an edge between them. Remember that as $\chi$ is dependent on $G_{aug}, \theta_{f_\tau}, \theta_{\tau}$, this loss (and other losses stated later) depends on them too.

\paragraph{Defender Loss}
For the defender, constructing the loss is non-trivial, as we have to integrate different competing objectives. We construct four different losses for aggregating the \textit{Defender loss}. 

\noindent \textbf{(1) Distribution Alignment Loss.} First, our goal is that the perturbed posterior distribution from the augmented model must be quite different from the original GNN's posterior distribution so that a similarity-metric-based attacker gets confused. So, we have to integrate a loss, which will try to maximize the difference between the posteriors $\mathbf{\chi}_u, \mathbf{\chi}_v$ for the nodes $u,v$, respectively, in an existent edge $(u,v)$.
The \textit{Distribution Alignment loss for Defender} is defined as 

{\small
\begin{align}
\begin{split}
    \mathcal{L}_{dist}(\mathcal{D}_{Q}, \chi) = \mathcal{L}_{dist}(\mathcal{D}_{Q}, G_{aug}; \theta_{f_\tau}, \theta_{\tau})\\
    := \mathbb{E}_{(u,v,y=1) \sim \mathcal{D}_Q}[-ShanonDiv( \mathbf{\chi}_u || \mathbf{\chi}_v )]
\end{split}\label{eqn:distloss}
\end{align}
}

Here, $ShanonDiv(.)$ is the Jensen-Shanon Divergence~\cite{shanon}, which is a symmetric version of the Kullback-Leibler (KL) Divergence. We use Jensen-Shanon instead of KL, to make the distribution loss symmetric for the edge endpoints. 

\noindent \textbf{(2) Correlation Distance Loss.} Second, to deviate the perturbed posterior distribution from the augmented model more from the original GNN's posterior distribution, we integrate another loss, which will try to minimize the correlation similarity between the posteriors $\mathbf{\chi}_u, \mathbf{\chi}_v$ for the nodes $u,v$, respectively, in an existent edge $(u,v)$. Whereas the $\mathcal{L}_{dist}$ is optimized to make the posterior vectors dissimilar in probability simplex, this correlation-based loss is optimized for the same task in the vector space. 
The \textit{Correlation Distance loss for Defender} is defined as: 

{\small
\begin{align}
\begin{split}
    \mathcal{L}_{corr}(\mathcal{D}_{Q}, \chi) = \mathcal{L}_{corr}(\mathcal{D}_{Q}, G_{aug}; \theta_{f_\tau}, \theta_{\tau})\\
    := \mathbb{E}_{(u,v,y=1) \sim \mathcal{D}_Q}[-\{1-\frac{\mathbf{\bar{\chi}}_u \cdot \mathbf{\bar{\chi}}_v}{{\lvert\lvert (\mathbf{\chi}_u - {\mathbf{\bar{\chi}}_u)} \rvert\rvert}_2{\lvert \lvert(\mathbf{\chi}_v - {\mathbf{\bar{\chi}}_v)} \rvert \rvert}_2}\}]
\end{split}\label{eqn:corrloss}
\end{align}
}

Here, we have described the formula for adjusted cosine similarity which is used interchangeably with correlation similarity. To convert it into a distance, we deduct the quantity from 1.

\noindent \textbf{(3) Alignment Calibration Loss.} Thirdly, while optimizing for the $\lambda_{dist}$, the original posterior can be perturbed in such a way that the main task can be derailed, i.e., misclassification happens. This happens if, after the perturbation, the ground truth label class's probability is not higher than the maximum one. So, we have to integrate an alignment-based loss, too, which ensures that only the class probabilities not representing the ground truth class are perturbed. As the attacker's aim is not the classification of nodes, this does not decrease much the strength of the defense.
Therefore, \textit{Alignment Calibration loss for defender} is defined as

{\small
\begin{align}
   \begin{split}
       \mathcal{L}_{align}(\mathcal{D}_{Q}, \mathbf{Y}, \mathbf{\chi})=\mathcal{L}_{align}(\mathcal{D}_{Q}, G_{aug}; \theta_{f_\tau}, \theta_{\tau})\\:= \mathbb{E}_{(u,v,y) \sim \mathcal{D}_Q}[-log({[\mathbf{\chi}_u}]_{\mathbf{Y}_{u}})
    - log({[\mathbf{\chi}_v}]_{\mathbf{Y}_{v}})]\end{split}\label{eqn:alignloss}
\end{align}
}

\noindent \textbf{(4) Misclassification loss.} Last, remember that merely perturbing the posterior vector is not our aim. We have to ensure that the attacker cannot predict the links' existence. So, we have to perturb in such a way that the posteriors from the model fool the surrogate model enough to misclassify the links.  So, we define the \textit{misclassification loss by the attacker for defender} as: 

{\small
\begin{align}\begin{split}
\mathcal{L}_{miss}(\mathcal{D}_{Q}, \chi; \theta_{g})=\mathcal{L}_{miss}(\mathcal{D}_{Q}, G_{aug}; \theta_{g},\theta_{f_\tau}, \theta_{\tau})\\:=\mathbb{E}_{(u,v,y) \sim \mathcal{D}_Q}[-log(1- 
    g(u,v, k(\mathbf{\chi}, \mathcal{D}_{Q}); \theta_g)_y)] 
    \label{eqn:missloss}
\end{split}\end{align}
}%

Finally, taking the linear combination of the losses, we have the \textit{Defender Loss},

{\small
\begin{align}\begin{split}
   \mathcal{L}_{defense}(\mathcal{D}_{Q}, G_{aug}; \theta_{g},\theta_{f_\tau}, \theta_{\tau})\\:=\begin{bmatrix}
\lambda_{miss}\\ \lambda_{align}\\ \lambda_{dist} \\ \lambda_{corr}\end{bmatrix}^T \begin{bmatrix}
    \mathcal{L}_{miss}(\mathcal{D}_{Q}, G_{aug}; \theta_{g},\theta_{f_\tau}, \theta_{\tau})\\
    \mathcal{L}_{align}(\mathcal{D}_{Q}, G_{aug}; \theta_{f_\tau}, \theta_{\tau})\\
    \mathcal{L}_{dist}(\mathcal{D}_{Q}, G_{aug}; \theta_{f_\tau}, \theta_{\tau})\\
    \mathcal{L}_{corr}(\mathcal{D}_{Q}, G_{aug}; \theta_{f_\tau}, \theta_{\tau})
\end{bmatrix} 
\end{split}\label{eqn:defloss}
\end{align}
}

where $\lambda_{miss}, \lambda_{align},\lambda_{dist}, \lambda_{corr}$ are hyperparameters to be set on.

\paragraph{Tri-Optimization Based Learning Loop}\label{sec:trioptloop} At first, we need to learn the parameters $\theta_{f_\tau}$ for the primary task of node classification. Then we learn the surrogate model parameters $\theta_g$ to measure up the optimal performance of the attacker, which the defense has to match. After this, we learn the augmentation parameters $\theta_\tau$ to step up against the optimized attacker. 
The tri-optimization objective can be formalized as follows: 
{\small
\begin{align}
\begin{split}
    \operatorname{argmin}_{\theta_{\tau}} \mathcal{L}_{defense}(\mathcal{D}_{Q},G_{aug};\theta_g^*,\theta_{f_\tau}^*, \theta_{\tau})\\
    s.t. \theta_{g}^* = \operatorname{argmin}_{\theta_{g}} \mathcal{L}_{attack}(\mathcal{D}_{Q},G_{aug};\theta_{g},\theta_{f_\tau}^*, \theta_{\tau})\\
    s.t. \theta_{f_\tau}^* = \operatorname{argmin}_{\theta_{f_\tau}} \mathcal{L}_{class}(G_{aug};\theta_{f_\tau}, \theta_{\tau})
\end{split}\label{eqn:trioptobj}
\end{align}
}
\Cref{func:runtriopt}  shows the procedure for updating the model parameters. We change the learning step a bit from what was stated before. We update the node augmentation parameters $\theta_\tau$ twice: once after classification (different from \Cref{sec:triopt-summary}, ``semi''-update) and again at the last step after the surrogate model (as stated in \Cref{sec:triopt-summary}, ``final''-update). As the node augmentation parameters also take part in the classification task, the loss classification gradients should also be 
propagated through them. Otherwise, it would downgrade the classification performance of the augmented model. We also pass some epoch-related hyperparameters to run the feature (semi and final) and surrogate updates. Besides,
as a learnable feature layer abstraction, we introduce a vector of the same size as $\text{\textit{number of new nodes}} \times \text{\textit{Feature Dimension}}$. The augmented features are always set as zero. So, before passing to the GNN, we add the learnable features to the fixed augmented zero features. Then, they are concatenated with the fixed original node features before being passed to the model. Hence, backpropagation is done during the addition operation of the learnable layer. The whole optimization loop is shown in \Cref{fig:nargis_triopt}.

\begin{func}[]
\caption{runTriOptimizationBasedFeatureLearning$(G_{aug}, \mathcal{D}_Q)$}\label[func]{func:runtriopt}
\renewcommand{\arraystretch}{1}
 \fontsize{8}{8}\selectfont
\SetKwInOut{Input}{Input}
\SetKwInOut{Output}{Output}
\SetKwInOut{Parameter}{Parameter}
\SetKwComment{Comment}{/* }{ */}
\RestyleAlgo{ruled}
\Input{Augmented Graph $G_{aug} = (V_{aug}, \textbf{X}_{aug}, \textbf{Y}_{aug}, \textbf{A}_{aug}, n_{aug}, c)$ with no features learned yet;\\Surrogate Query Edge Dataset $\mathcal{D}_{Q}$;\\}
\Parameter{number of epochs $\eta_{outer}, \eta_{class}, \eta_{surr}, \eta_{def}$;\\}
\Output{Learned Augmented Graph, $G_{aug}^{\tau} = (V_{aug}, \textbf{X}_{aug}^{\tau}, \textbf{Y}_{aug}, \textbf{A}_{aug}, n_{aug}, c)$;}
Initiate $\theta_g^{0},  \theta_{f_\tau}^{0}, \theta_{\tau}^{0}$\;
\For{$t=0$ to $\eta_{outer}-1$}{
\Comment{\textbf{Stage 1}: Learn Classification Model}
${{\hat{\theta}}_{f_\tau}^{0}}, {\hat{\theta}}_{\tau}^{0} \gets \theta_{f_\tau}^{t}, \theta_{\tau}^{t}$ \;
\For{$t_{c}=0$ to $\eta_{class}-1$}{
Calculate $\mathcal{L}_{class}(G_{aug}; {{\hat{\theta}}_{f_\tau}^{t_{c}},} {{\hat{\theta}}_{\tau}^{t_c})}$ using \Cref{eqn:classloss}\;
$ {{\hat{\theta}}_{f_\tau}^{t_{c}+1}}, {{\hat{\theta}}_{\tau}^{t_{c}+1}} \gets $ Update using back-prop from $\mathcal{L}_{class}(G_{aug}; {{\hat{\theta}}_{f_\tau}^{t_{c}},} {{\hat{\theta}}_{\tau}^{t_c})}$\;}
\Comment{\textbf{Feature Learner Semi-Update}}
${\theta_{f_\tau}^{t+1}}, {\theta_{\tau}^{t+\frac{1}{2}}}  \gets {{\hat{\theta}}_{f_\tau}^{\eta_{class}}},{{\hat{\theta}}_{\tau}^{\eta_{class}}} $\; 
\If{$t == 0$}{\Comment{\textbf{Stage 2}: Learn Surrogate Attacker Model}
${{\hat{\theta}}_{g}^{0}} \gets \theta_{g}^{t}  $ \;

\For{$t_{s}=0$ to $\eta_{surr}-1$}{
Calculate $\mathcal{L}_{attack}(\mathcal{D}_{Q},G_{aug};{{\hat{\theta}}_{g}^{t_s}},\theta_{f_\tau}^{t+1}, \theta_{\tau}^{t+\frac{1}{2}})$ using \Cref{eqn:lossatk}\;
${{\hat{\theta}}_{g}^{t_s+1}}\gets $ Update using back-prop from $\mathcal{L}_{attack}(\mathcal{D}_{Q},G_{aug};{{\hat{\theta}}_{g}^{t_s}},\theta_{f_\tau}^{t+1}, \theta_{\tau}^{t+\frac{1}{2}})$\;}
\Comment{\textbf{Surrogate Attacker Update}}
$\theta_{g}^{t+1} \gets{{\hat{\theta}}_{g}^{\eta_{surr}}}$\;}
\Else{$\theta_{g}^{t+1} \gets \theta_{g}^{t}$}

\Comment{\textbf{Stage 3} Learn Augmented Features}
${{\hat{\theta}}_{\tau}^{0}}\gets \theta_{\tau}^{t+\frac{1}{2}} $\;
\For{$t_{d}=0$ to $\eta_{def}-1$}{
Calculate $\mathcal{L}_{defense}(\mathcal{D}_{Q}, G_{aug};\theta_{g}^{t+1},\theta_{f_\tau}^{t+1},{{\hat{\theta}}_{\tau}^{t_d}})$ using \Cref{eqn:defloss}\;
${{\hat{\theta}}_{\tau}^{t_d+1}}\gets $ 
Update using back-prop from 
$\mathcal{L}_{defense}(\mathcal{D}_{Q}, G_{aug};\theta_{g}^{t+1},\theta_{f_\tau}^{t+1},{{\hat{\theta}}_{\tau}^{t_d}})$\;}
\Comment{\textbf{Feature Learner Final Update}}
$\theta_{\tau}^{t+1} \gets{{\hat{\theta}}_{\tau}^{\eta_{def}}}$\;

}
\Comment{$\theta_{\tau}^{\eta_{outer}}$ is the learned augmented node features}
$\mathbf{X}_{aug}^{\tau} \gets $ replace $\mathbf{X}_{new}$ from $\mathbf{X_{aug}}$ in \Cref{eqn:xaugaugyaug} with $\theta_{\tau}^{\eta_{outer}}$\;
\Return{$G_{aug}^{\tau} = (V_{aug}, \textbf{X}_{aug}^{\tau}, \textbf{Y}_{aug}, \textbf{A}_{aug}, n_{aug}, c)$}

\end{func}

\subsection{Algorithm: \mymodel{}}
Combining all the modules and functions stated above, we train \mymodel{}. The training algorithm has been shown in \Cref{alg:nodeaug} (and also for reference in \Cref{fig:nargis_overview}). 
\begin{algorithm}[!htp]
\renewcommand{\arraystretch}{1}
 \fontsize{8}{8}\selectfont
\renewcommand{\arraystretch}{1}
 \fontsize{8}{8}\selectfont
\SetKwInOut{Input}{Input}
\SetKwInOut{Output}{Output}
\SetKwInOut{Parameter}{Parameter}
\SetKwComment{Comment}{/* }{ */}
\SetKwFunction{getaug}{getAugmentedGraph}
\SetKwFunction{gensur}{generateSurrogateEdgeQueryDataset}
\SetKwFunction{triopt}{runTriOptimizationBasedFeatureLearning}
\SetKwFunction{perturb}{learnPerturbedGraphEmbedding}
\RestyleAlgo{ruled}

\Input{Graph $G = (V, \textbf{X}, \textbf{Y}, \textbf{A}, n, c)$\;\\} 
\Parameter{Number of new nodes to augment, $n_{new}$;\\Set Threshold $\zeta$;\\number of epochs $\eta_{outer}, \eta_{class}, \eta_{surr}, \eta_{def}$\;}
\Output{Perturbed Node Posteriors, $\Phi \in [0,1]^{n\times c}$}
\Comment{\textbf{Augmentation Module}-\Cref{func:getAugGraph} }
Augmented Graph, $G_{aug} \gets $\getaug$(G)$\;

Training Edge Dataset, $\mathcal{D}_{Q} \gets $ Edges, sampled negative edges and edge labels from $(G)$\;
\Comment{\textbf{Tri-Optimization based Feature Learning Module}-\Cref{func:runtriopt} }
Augmented Graph with Learned Features, $G_{aug}^{\tau}\gets$ \triopt$(G_{aug}, \mathcal{D}_Q)$\;
\Comment{\textbf{Get Perturbed Posteriors from the Defense ingrained GNN}}
$\Phi_{aug} \gets f_{\tau}(G_{aug}^{\tau})$\;
\Comment{\textbf{Return only the main graph part}}
Perturbed (Main Graph) Node Posteriors, $\Phi \gets \Phi_{aug}[ :n, :]$\;
\Return{$\phi$}
\caption{\mymodel{}}\label{alg:nodeaug}
\end{algorithm}

\section{Evaluation}\label{sec:exp} In this section, we describe the datasets, models, attacks and metrics to evaluate our model's performances.
\subsection{Experimental Setup} 

\subsubsection{Datasets} We evaluate our models and attacks on three popular Graph datasets: Cora, Citeseer, Pubmed \cite{yang2016revisiting}. These datasets are citation datasets used for benchmarking GNN models \cite{kipf2016semi}. The nodes of the datasets represent publications, and links denote their citations. Cora and Citeseer node features are 0/1 vectors representing the absence/presence of a particular word or tag in a fixed dictionary. Pubmed features are weighted TF-IDF vectors for a vocabulary.

\noindent \textbf{Dataset Configuration for Our Learning:} For training each dataset for node classification, we have followed the train, validation, test split settings from \cite{kipf2016semi}, where (for example,) for training the Cora Dataset only 20 nodes per class were used, whereas the dataset has 2,708 nodes. While training \mymodel{} and evaluating the attacks, where the edges form an important part of the training, we at first split the graph based on 70-10-20\% train-validation-split of the nodes, and then on the split graph, perform edge split on the positive edges with the same ratio. As stated in \cite{kipf2016semi, he2021stealing}, we sampled the same number of negative edges for each split. We trained \mymodel{} on the train edges (positive and negative), validated on validation edges (positive and negative), and tested the attacks on test split edges (positive and negative). All the graphs are trained on transductive settings, where some of the nodes are labeled and training is done on them, while node classification is done on the rest unlabeled nodes.

\subsection{Software and Hardware Settings} We have used Pytorch (\cite{paszke2019pytorch}), Pytorch Geometric 
(\cite{fey2019fast}) and NetworkX (\cite{hagberg2008exploring}) for the implementation. We ran our experiments on a server with 3 NVIDIA GeForce RTX 3090 GPUs, totaling 72 GB.

\subsection{Node Classification GNN models and hyperparameters} As GNN node classification model, we have used three message-passing GNNs for evaluation: GCN \cite{kipf2016semi}, GAT \cite{veličković2018graph} and SAGE  \cite{hamilton2017inductive}. All of the models have two convolution layers with their respective message-passing mechanisms. The dimension of the convolution layers are $16$. We also used dropout layers after each convolution layer with $p=0.5$. The first convolution layers had ReLU activation, and the last convolution layer, which is also the output layer, had softmax activation. Every GNN were trained for 200 epochs with ADAM 
\cite{adam} optimizer, with learning rate 0.005 and weight decay rate of $5 \times 10 ^{-4}$. After every 10 epochs, we ran validation to find the validation loss, and the model with the least validation loss was saved.

\subsection{Cluster Numbers}
Previously, in \Cref{sec:chalover}, we have discussed that the cluster number has to be inversely proportional to the Graph Density. We have calculated the graph density of Cora, Citeseer and Pubmed as $ 0.00144, 0.000823, 0.000228$, respectively. Their inverse ratios are: $1:1.75:6.3 = 10:17.5:63$. So we set cluster numbers (nearby rounded as multiples of 10) 10, 20, 60 for Cora, Citeseer and Pubmed, respectively.

\subsection{{Tri-Optimization Based Module's Hyperparameters}} For every GNN model with basic (unguarded) settings, the corresponding augmented GNN model also had the same configuration and convolution layers. The number of epochs $\eta_{outer}, \eta_{class}, \eta_{def}$ in \Cref{func:runtriopt} in \Cref{sec:methodology}, are set as $10, 200,50$. For surrogate learning, we set the batch size of 512 for training edge set and set $\eta_{surr}$ as the number of epochs needed to learn the whole set of training edges. Classification, surrogate, and defense Learning models are optimized using Adam \cite{adam} with learning rate 0.01, 0.001 and 0.001, respectively. Both classification and defense learning have weight decay of $5\times 10 ^{-4}$. The surrogate attacker model is a learnable matrix of dimension $\text{\textit{class number}} \times \text{\textit{hidden dimension}}$ (set as 10). We multiply the posteriors from the endpoints with the matrix to get two vectors, and then we take the sigmoid of their dot product.
The loss hyperparameters $\mathcal{\lambda}_{misc}$, $ \mathcal{\lambda}_{align}$, $\mathcal{\lambda}_{dist}$ and $ \mathcal{\lambda}_{corr}$ are set as 4, 0.8, 2 and 0.6 through grid search, respectively. 

\subsection{Attack Models}
For the Attack Models, we have used the same settings as described in \cite{he2021stealing}, for supervised attacker, reference models, shadow GNN and shadow reference models. The supervised attacker Model is an MLP with three hidden layers of 32 dimension, which is trained for 50 epochs with learning rate 0.005 with Adam. Both the reference and shadow reference model for the Attacker are MLPs with two hidden layers with 64 and 32 nodes, which is trained for 50 epochs with Adam optimizer of learning rate 0.01 and weight decay as same. The main and shadow reference models are used for the formation of features for the edges from the posteriors of the nodes from the main and shadow dataset, respectively. The shadow GNN model for learning on shadow dataset is a GCN, which has the same architecture as the classification model, except the hidden dimension is 32 and it is trained for 100 epochs. Shadow GNN is trained to apply transfer learning in four forms of attacks (Attack-1, 4, 5 and 7).

\subsection{How to Interpret the Experiments}
Combining \Cref{tab:single,,tab:transfer}, we have discussed about ten tasks: (1) Node prediction, (2-9) eight forms of Link Stealing Attacks, (10) LinkTeller Attack. For the first task, we want the performance under the defense model to be decreased as low as possible. For the rest of the nine tasks, we want the performance of the attack against the defense model to be decreased as much as possible.

A good defensive model has higher node prediction accuracy (indicating lower fidelity loss) and lower AUC (higher attack thwarted).

We also calculated a loss concerning each metric's Baseline (Undefended) model. For accuracy, the loss w.r.t. the baseline model should be as low as possible. For AUC, the loss w.r.t. the baseline model should be as high as possible. 

The models marked with a long cross \cross{} are tuned for the highest fidelity/node prediction accuracy. So, their tradeoff is lower privacy (Higher AUC). In contrast, the models marked with an asterisk (*) are tuned for the highest privacy (lower AUC), where the tradeoff is possibly lower fidelity. 

\subsection{Metrics} For the node prediction task, we have used accuracy as our performance metric. For link stealing task evaluation, as per the works of \cite{he2021stealing, wu2022linkteller, kolluri2022lpgnet, zhu2023blink}, we have used AUC (Area under ROC Curve) metric. This metric is used for two reasons: (1) it is safe from class imbalance issue, unlike precision or recall (2) For unsupervised classifications where no labels are present, and we have to decide the label comparing with a threshold, AUC can denote the probability that, under any set threshold, a randomly selected positive edge node pair will have higher probability than a randomly selected negative edge node pair. If all node pairs are ranked randomly, AUC will be 0.5. Each evaluation were run three times with same seed for reproducibility, and the mean value was reported. Also, we have reported the accuracy and AUC loss compared to the Basic (Undefended) model. Moreover, as Attack-0  settings have no knowledge available for the attackers, it is considered as the most trivial form of attack. Hence, our tunings were done considering this trivial attack's performance, as the models are at least expected to defend Attack-0 , the most trivial one. 

\subsection{Baseline Models} For baseline evaluation, we have chosen two differential privacy (DP) based defense models, \textbf{EdgeRand} and \textbf{LapGraph}, from \cite{wu2022linkteller}. We have chosen three versions of them, with varying DP parameter $\epsilon$ (lower value indicates higher privacy, higher value indicates higher fidelity): \erf{}, \lapf{} and \lapd. It is to be noted that, as these models have trade-offs between fidelity and privacy, for the sake of comparison we will consider both cases. In each configuration of a particular GNN model and a dataset, the DP model marked with \cross{} is the DP model with highest prediction accuracy (denoted onwards as \dphfp: DP Defense with Highest Fidelity Preferred). Also, the DP model marked with \aster{} is the one with lowest Attack-0 AUC (denoted onwards as \dphpp: DP Defense with Highest Privacy Preferred). For example, in \Cref{tab:single} for Cora dataset with SAGE GNN model, \erf{} is the \dphfp{} and \lapd{} is the \dphpp.

\subsection{Tuned \mymodel{}}Considering fidelity-privacy tradeoff, we have evaluated two more versions of \mymodel{} whose $\mathcal{\lambda}_{align}$ are different from the original one. Whereas the original \mymodel{}'s hyperparameters were tuned for highest fidelity, \mymodeldef{}'s $\mathcal{\lambda}_{align}$ is adjusted to match the concerning \dphpp's Attack-0  AUC performance. Additionally, \mymodelpred{}'s $\mathcal{\lambda}_{align}$ is adjusted to match the concerning \dphfp's Prediction performance. They are also marked in the tables with \aster{} and \cross{}, respectively. The $\mathcal{\lambda}_{align}$ for each dataset and GNN combinations are described in \Cref{tab:align_tune}.

\begin{table}[!htp]
\fixtableonecol{6}
\caption{Tuned Values for $\mathcal{\lambda}_{align}$ for $\mymodel{}$-PredTuned and $\mymodel{}$-DefTuned}
   \label{tab:align_tune}
    \centering
    \begin{tabular}{|c|c|c|c|}
       \textbf{GNN} & \textbf{Dataset}  & \begin{tabular}[c]{@{}c@{}}$\mathcal{\lambda}_{align}$ tuned \\ for Highest \\Node Prediction\\($\mymodel{}$-PredTuned)\end{tabular} & \begin{tabular}[c]{@{}c@{}}$\mathcal{\lambda}_{align}$ tuned\\for Lowest\\Attack:0 AUC \\($\mymodel{}$-DefTuned)\end{tabular}  \\ \bigline
        
        SAGE & Cora & 0.3 & 0.25\\
      SAGE & Citeseer  & 0.001 & 0.005 \\
     SAGE & Pubmed  & 0.001 & 0.002\\ \bigline
     GCN & Cora & 2.0 & 1.0 \\
      GCN & Citeseer  & 0.5 & 0.001 \\
     GCN & Pubmed  & 0.15 & 0.1 \\ \bigline
     GAT & Cora & 0.15 & 0.25 \\
      GAT & Citeseer  & 0.002& 0.001\\
     GAT & Pubmed  & 0.001& 0.002\\ \bigline
     
    \end{tabular}
    
\end{table}

\subsection{Performance Criteria} We have considered the Node Prediction Performance (Accuracy) and the AUC of all the attacks (0-7) stated in \cite{he2021stealing} (further described in \Cref{tab:attack_config_intro}) and LinkTeller attack from \cite{wu2022linkteller}. We have shown Node prediction, Attack-0, 2, 3, 6; and LinkTeller performances together in \Cref{tab:single}, as there is no shadow dataset involved. For shadow dataset based attacks (Attack- 1,4,5,7), we have shown the results in \Cref{tab:transfer}. Moreover, as stated in \cite{he2021stealing}, Attack-0,2,3,6 had different feature design schemes. For comparison, we chose the best feature design for each attacks: Correlation distance, Posterior Distance, all features combined and all features combined (again), respectively.

\subsection{Performance Analysis: Node Prediction} For almost all dataset and GNN combinations, best node prediction accuracies were achieved by \mymodel{}. Besides, the best accuracy performances among other defenses in these combinations were achieved by the tuned-versions of \mymodel{}. In GCN:Citeseer combination, the highest accuracy was achieved by \mymodelpred{}. It is to be noted that only in GAT:Citeseer, the second best performance was achieved by a \dphfp.

\subsection{Performance Analysis: Singular Attacks (Attack= 0, 2, 3, 6)} For Attack-0, in all datasets, \mymodel{}'s tuned variants have achieved best performance for SAGE; but for GCN and GAT, it is the \dphpp{} who defended the best. In case of Attack-2, the finding is the same (\cite{he2021stealing} used the correlation difference between posteriors as the best feature for Attack-2 which is same as Attack-0, hence the same performance). Moreover, for Attack-3, none of the variants of \mymodel{} has been able to beat \dphpp. Finally, for Attack-6, \dphpp{} has beaten others in most combinations, while being the second best in cases where other DP based methods got the best result. At first glance, it might seem that \mymodel{} or its variants are not a good form of defense. But in hindsight, while considering the node prediction accuracy (fidelity), they pose as very balanced defense methods. For most of the cases where \mymodel{} and its variants came second or third-best, their defense performance were not that far, and the fidelities were better. 

\subsection{Performance Analysis: LinkTeller}
For LinkTeller (\cite{wu2022linkteller}), best performances were achieved by the DP defenses designed to beat it. \mymodel{} or its variants could not defend it better as LinkTeller focuses on the nodes' influence on each other more rather than exploiting their posterior proximities. As our model has not included inter-node influence-based losses in its design, it has not been able to show near-optimal performance for LinkTeller.

\subsection{Performance Analysis: Transfer Attacks (Attack- 1, 4, 5, 7)} For SAGE based GNN model, \mymodel{}'s tuned versions have shown superior performances for all dataset and shadow dataset combinations for Attack-1,4,5,7. Even in Pubmed $\rightarrow$ Citeseer combination for Attack-7, the performance of tuned \mymodel{}'s are close to the \dphpp{} one. But for the GCN-based GNN model, it is always the DP-based models that show the best defensive performances. For GAT-based models, the performance is almost similar, except for some optimal performances for Attack-7.

\begin{table}[]
\centering
\caption{Defence Performances on Node Prediction and all singular Attacks (Attack-0, 2, 3, 6; LinkTeller). DP-based Defense and \mymodel{} Models marked with $\dag$ and $\boldsymbol{*}$ are tuned, respectively, for high utility (Higher Node Prediction Accuracy) and high privacy (lower Attack-0 AUC).}\label{tab:single}
\fixtableonecol{6}
\begin{tabular}{cccccccccccccc}
            \bigline
\textbf{}    & \textbf{Dataset}                          & \multicolumn{12}{c}{Cora}                                                                                                                                                                                                                                                                                                                                                                                                                                                                                                                                \\\bigline

\textbf{GNN} & \textbf{Defense}                          & \multicolumn{2}{c}{\textbf{\begin{tabular}[c]{@{}c@{}}Node \\ Prediction\end{tabular}}}   & \multicolumn{2}{c}{\textbf{\begin{tabular}[c]{@{}c@{}}Attack-0:\\ Correlation\end{tabular}}} & \multicolumn{2}{c}{\textbf{\begin{tabular}[c]{@{}c@{}}Attack-2:\\ Posterior\end{tabular}}} & \multicolumn{2}{c}{\textbf{\begin{tabular}[c]{@{}c@{}}Attack-3:\\ All\end{tabular}}}  & \multicolumn{2}{c}{\textbf{\begin{tabular}[c]{@{}c@{}}Attack-6:\\ All\end{tabular}}}  & \multicolumn{2}{c}{\textbf{Linkteller}}                                              \\\bigline
             &                                           & \textbf{\rotatebox{90}{Acc$\uparrow$}} & \textbf{\rotatebox{90}{Loss$\downarrow$}} & \textbf{\rotatebox{90}{AUC$\downarrow$}}     & \textbf{\rotatebox{90}{Loss$\uparrow$}}    & \textbf{\rotatebox{90}{AUC$\downarrow$}}    & \textbf{\rotatebox{90}{Loss$\uparrow$}}   & \textbf{\rotatebox{90}{AUC$\downarrow$}} & \textbf{\rotatebox{90}{Loss$\uparrow$}} & \textbf{\rotatebox{90}{AUC$\downarrow$}} & \textbf{\rotatebox{90}{Loss$\uparrow$}} & \textbf{\rotatebox{90}{AUC$\downarrow$}} & \textbf{\rotatebox{90}{Loss$\uparrow$}} \\ \bigline
             & Basic                                     & 0.787                                        & -                                          & 0.92                                          & -                                           & 0.92                                         & -                                          & 0.93                                      & -                                        & 0.93                                      & -                                        & 0.98                                      & -                                        \\
& (\dag) EdgeRand$(\epsilon=6)$   & 0.564                                        & 0.223                                      & 0.81                                          & 0.112                                       & 0.81                                         & 0.112                                      & 0.83                                      & 0.096                                    & 0.84                                      & 0.087                                    & 0.5                                       & \textbf{0.476}                                    \\
             &($\boldsymbol{*}$) LapGraph$(\epsilon=0.1)$ & 0.417                                        & 0.37                                       & 0.7                                           & 0.226                                       & 0.7                                          & 0.226                                      & 0.72                                      & \textbf{0.207}                                    & 0.79                                      & \textbf{0.138}                                    & 0.5                                       & 0.474                                    \\
  SAGE           & LapGraph$(\epsilon=6)$   & 0.399                                        & 0.388                                      & 0.71                                          & 0.213                                       & 0.71                                         & 0.213                                      & 0.74                                      & 0.185                                    & 0.8                                       & 0.124                                    & 0.61                                      & 0.369                                    \\
             & NARGIS                                    & 0.81                                         & \textbf{-0.023           }                          & 0.86                                          & 0.067                                       & 0.86                                         & 0.067                                      & 0.86                                      & 0.069                                    & 0.88                                      & 0.052                                    & 0.97                                      & 0.01                                     \\
             & ($\boldsymbol{*}$)NARGIS-DefTuned                           & 0.513                                        & 0.274                                      & 0.65                                          & \textbf{0.272}                                       & 0.65                                         & \textbf{0.272}                                      & 0.76                                      & 0.168                                    & 0.81                                      & 0.113                                    & 0.96                                      & 0.013                                    \\
             & (\dag) NARGIS-PredTuned                          & 0.604                                        & \textit{0.183 }                                     & 0.69                                          & 0.23                                        & 0.69                                         & 0.23                                       & 0.79                                      & 0.14                                     & 0.82                                      & 0.102                                    & 0.96                                      & 0.013                                    \\ \dottedline{1}{14} \\            & Basic                                     & 0.793                                        & -                                          & 0.94                                          & -                                           & 0.94                                         & -                                          & 0.93                                      & -                                        & 0.93                                      & -                                        & 0.93                                      & -                                        \\
             &(\dag) EdgeRand$(\epsilon=6)$   & 0.64                                         & 0.153                                      & 0.88                                          & 0.057                                       & 0.88                                         & 0.057                                      & 0.89                                      & 0.044                                    & 0.86                                      & 0.071                                    & 0.5                                       & \textbf{0.43}                                     \\
             & ($\boldsymbol{*}$)LapGraph$(\epsilon=0.1)$ & 0.144                                        & 0.649                                      & 0.52                                          & \textbf{0.416}                                       & 0.52                                         & \textbf{0.416}                                      & 0.54                                      & \textbf{0.39}                                     & 0.75                                      & \textbf{0.179}                                    & 0.5                                       & 0.425                                    \\
GCN            & LapGraph$(\epsilon=6)$   & 0.302                                        & 0.491                                      & 0.67                                          & 0.272                                       & 0.67                                         & 0.272                                      & 0.69                                      & 0.245                                    & 0.79                                      & 0.137                                    & 0.59                                      & 0.336                                    \\
             & NARGIS                                    & 0.852                                        & \textbf{-0.059}                                     & 0.9                                           & 0.039                                       & 0.9                                          & 0.039                                      & 0.89                                      & 0.037                                    & 0.9                                       & 0.03                                     & 0.82                                      & 0.109                                    \\
             & ($\boldsymbol{*}$) NARGIS-DefTuned                           & 0.763                                        & \textit{0.03}                                       & 0.89                                          & 0.047                                       & 0.89                                         & 0.047                                      & 0.88                                      & 0.05                                     & 0.89                                      & 0.034                                    & 0.79                                      & 0.139                                    \\
             & (\dag) NARGIS-PredTuned                          & 0.746                                        & 0.047                                      & 0.89                                          & 0.05                                        & 0.89                                         & 0.05                                       & 0.89                                      & 0.044                                    & 0.9                                       & 0.033                                    & 0.78                                      & 0.146                                    \\\dottedline{1}{14} \\ 
             & Basic                                     & 0.748                                        & -                                          & 0.92                                          & -                                           & 0.92                                         & -                                          & 0.91                                      & -                                        & 0.92                                      & -                                        & 0.8                                       & -                                        \\
             & (\dag) EdgeRand$(\epsilon=6)$   & 0.399                                        & 0.349                                      & 0.73                                          & 0.188                                       & 0.73                                         & 0.188                                      & 0.74                                      & 0.174                                    & 0.8                                       & 0.111                                    & 0.5                                       & \textbf{0.296}                                    \\
            & ($\boldsymbol{*}$) LapGraph$(\epsilon=0.1)$ & 0.125                                        & 0.623                                      & 0.52                                          & \textbf{0.4}                                         & 0.52                                         & 0.4                                        & 0.54                                      & \textbf{0.368}                                    & 0.75                                      & \textbf{0.165}                                    & 0.51                                      & 0.289                                    \\
GAT             & LapGraph$(\epsilon=6)$   & 0.217                                        & 0.531                                      & 0.61                                          & 0.307                                       & 0.61                                         & 0.307                                      & 0.62                                      & 0.288                                    & 0.77                                      & 0.146                                    & 0.59                                      & 0.209                                    \\
             & NARGIS                                    & 0.71                                         & \textbf{0.038}                                      & 0.81                                          & 0.105                                       & 0.81                                         & 0.105                                      & 0.83                                      & 0.078                                    & 0.87                                      & 0.044                                    & 0.71                                      & 0.089                                    \\
             & ($\boldsymbol{*}$) NARGIS-DefTuned                           & 0.329                                        & \textit{0.419}                                      & 0.65                                          & 0.264                                       & 0.65                                         & 0.264                                      & 0.74                                      & 0.17                                     & 0.84                                      & 0.081                                    & 0.61                                      & 0.183                                    \\
             & (\dag)NARGIS-PredTuned                          & 0.289                                        & 0.459                                      & 0.63                                          & 0.286                                       & 0.63                                         & 0.286                                      & 0.72                                      & 0.186                                    & 0.83                                      & 0.086                                    & 0.66                                      & 0.141                                    \\
             &                                           &                                              &                                            &                                               &                                             &                                              &                                            &                                           &                                          &                                           &                                          &                                           &                                          \\
                                                 \bigline
             & \textbf{Dataset}                          & \multicolumn{12}{c}{Citeseer} \\ \bigline                                                                                                                                                                                                                                                                                                                                                                                                                                                                                                                         
             
             & Basic                                     & 0.687                                        & -                                          & 0.93                                          & -                                           & 0.93                                         & -                                          & 0.94                                      & -                                        & 0.94                                      & -                                        & 0.99                                      & -                                        \\
             & EdgeRand$(\epsilon=6)$   & 0.5                                          & 0.187                                      & 0.8                                           & 0.123                                       & 0.8                                          & 0.123                                      & 0.82                                      & 0.123                                    & 0.86                                      & 0.084                                    & 0.5                                       & \textbf{0.489}                                    \\
             & LapGraph$(\epsilon=0.1)$ & 0.547                                        & 0.14                                       & 0.76                                          & 0.171                                       & 0.76                                         & 0.171                                      & 0.77                                      & 0.174                                    & 0.82                                      & \textbf{0.122 }                                   & 0.51                                      & 0.482                                    \\
SAGE             & ($\boldsymbol{*}$)(\dag )LapGraph$(\epsilon=6)$   & 0.553                                        & 0.134                                      & 0.75                                          & 0.182                                       & 0.75                                         & 0.182                                      & 0.76                                      & \textbf{0.181}                                    & 0.82                                      & 0.12                                     & 0.57                                      & 0.418                                    \\
             & NARGIS                                    & 0.584                                        & \textbf{0.103}                                      & 0.75                                          & 0.178                                       & 0.75                                         & 0.178                                      & 0.83                                      & 0.111                                    & 0.86                                      & 0.081                                    & 0.97                                      & 0.015                                    \\
             & ($\boldsymbol{*}$) NARGIS-DefTuned                           & 0.495                                        & 0.192                                      & 0.66                                          & \textbf{0.266}                                       & 0.66                                         & \textbf{0.266}                                      & 0.81                                      & 0.132                                    & 0.84                                      & 0.102                                    & 0.97                                      & 0.015                                    \\
             & (\dag) NARGIS-PredTuned                          & 0.56                                         & \textit{0.127}                                      & 0.77                                          & 0.159                                       & 0.77                                         & 0.159                                      & 0.88                                      & 0.058                                    & 0.89                                      & 0.054                                    & 0.94                                      & 0.054                                    \\\dottedline{1}{14} \\ 
             & Basic                                     & 0.679                                        & -                                          & 0.95                                          & -                                           & 0.95                                         & -                                          & 0.94                                      & -                                        & 0.95                                      & -                                        & 0.95                                      & -                                        \\
             & (\dag)EdgeRand$(\epsilon=6)$   & 0.46                                         & 0.219                                      & 0.85                                          & 0.102                                       & 0.85                                         & 0.102                                      & 0.85                                      & 0.09                                     & 0.86                                      & 0.092                                    & 0.5                                       & \textbf{0.446}                                    \\
             & ($\boldsymbol{*}$)LapGraph$(\epsilon=0.1)$ & 0.194                                        & 0.485                                      & 0.53                                          & \textbf{0.423}                                       & 0.53                                         & \textbf{0.423}                                      & 0.55                                      & \textbf{0.396}                                    & 0.8                                       & \textbf{0.148}                                   & 0.52                                      & 0.431                                    \\
GCN             & LapGraph$(\epsilon=6)$   & 0.262                                        & 0.417                                      & 0.64                                          & 0.31                                        & 0.64                                         & 0.31                                       & 0.65                                      & 0.29                                     & 0.81                                      & 0.136                                    & 0.57                                      & 0.373                                    \\
             & NARGIS                                    & 0.463                                        & 0.216                                      & 0.82                                          & 0.127                                       & 0.82                                         & 0.127                                      & 0.86                                      & 0.088                                    & 0.9                                       & 0.047                                    & 0.86                                      & 0.089                                    \\
             & ($\boldsymbol{*}$)NARGIS-DefTuned                           & 0.185                                        & 0.494                                      & 0.71                                          & 0.24                                        & 0.71                                         & 0.24                                       & 0.77                                      & 0.178                                    & 0.84                                      & 0.109                                    & 0.84                                      & 0.103                                    \\
             & (\dag) NARGIS-PredTuned                          & 0.577                                        & \textbf{\textit{0.102}}                                      & 0.9                                           & 0.054                                       & 0.9                                          & 0.054                                      & 0.9                                       & 0.043                                    & 0.92                                      & 0.03                                     & 0.84                                      & 0.11                                     \\ \dottedline{1}{14} \\ 
             & Basic                                     & 0.677                                        & -                                          & 0.93                                          & -                                           & 0.93                                         & -                                          & 0.92                                      & -                                        & 0.93                                      & -                                        & 0.87                                      & -                                        \\
             & (\dag) EdgeRand$(\epsilon=6)$   & 0.3                                          & \textit{0.377}                                      & 0.68                                          & 0.249                                       & 0.68                                         & 0.249                                      & 0.7                                       & 0.222                                    & 0.83                                      & 0.1                                      & 0.5                                       & 0.366                                    \\
             & ($\boldsymbol{*}$)LapGraph$(\epsilon=0.1)$ & 0.225                                        & 0.452                                      & 0.55                                          & \textbf{0.387}                                       & 0.55                                         & \textbf{0.387}                                      & 0.56                                      & \textbf{0.36}                                     & 0.8                                       & \textbf{0.132}                                    & 0.5                                       & \textbf{0.363}                                    \\
GAT             & LapGraph$(\epsilon=6)$   & 0.286                                        & 0.391                                      & 0.65                                          & 0.286                                       & 0.65                                         & 0.286                                      & 0.65                                      & 0.27                                     & 0.81                                      & 0.118                                    & 0.58                                      & 0.286                                    \\
             & NARGIS                                    & 0.626                                        & \textbf{0.051}                                      & 0.89                                          & 0.045                                       & 0.89                                         & 0.045                                      & 0.9                                       & 0.025                                    & 0.91                                      & 0.017                                    & 0.77                                      & 0.097                                    \\
             & ($\boldsymbol{*}$)NARGIS-DefTuned                           & 0.238                                        & 0.439                                      & 0.68                                          & 0.253                                       & 0.68                                         & 0.253                                      & 0.8                                       & 0.124                                    & 0.88                                      & 0.048                                    & 0.66                                      & 0.208                                    \\
             & (\dag) NARGIS-PredTuned                          & 0.295                                        & 0.382                                      & 0.74                                          & 0.189                                       & 0.74                                         & 0.189                                      & 0.83                                      & 0.091                                    & 0.9                                       & 0.03                                     & 0.61                                      & 0.254                                    \\
             &                                           &                                              &                                            &                                               &                                             &                                              &                                            &                                           &                                          &                                           &                                          &                                           &                                          \\
             \bigline
             & \textbf{Dataset}                          & \multicolumn{12}{c}{Pubmed} \\    \bigline                                                                                                                                                                                                                                                                                                                                                                                                                                                                                                                          
                                                     
             & Basic                                     & 0.767                                        & -                                          & 0.86                                          & -                                           & 0.86                                         & -                                          & 0.9                                       & -                                        & 0.9                                       & -                                        & 0.99                                      & -                                        \\
             & (\dag) EdgeRand$(\epsilon=6)$   & 0.705                                        & 0.062                                      & 0.76                                          & 0.107                                       & 0.76                                         & 0.107                                      & 0.77                                      & 0.13                                     & 0.81                                      & 0.092                                    & 0.5                                       & \textbf{0.49}                                     \\
             & LapGraph$(\epsilon=0.1)$ & 0.7                                          & 0.067                                      & 0.75                                          & 0.11                                        & 0.75                                         & 0.11                                       & 0.77                                      & 0.132                                    & 0.81                                      & 0.092                                    & 0.5                                       & 0.488                                    \\
 SAGE            & ($\boldsymbol{*}$)LapGraph$(\epsilon=6)$   & 0.7                                          & 0.067                                      & 0.75                                          & 0.113                                       & 0.75                                         & 0.113                                      & 0.77                                      & \textbf{0.135}                                    & 0.8                                       & \textbf{0.094}                                    & 0.5                                       & 0.489                                    \\
             & NARGIS                                    & 0.734                                        & \textbf{0.033}                                      & 0.82                                          & 0.039                                       & 0.82                                         & 0.039                                      & 0.87                                      & 0.033                                    & 0.87                                      & 0.031                                    & 0.98                                      & 0.011                                    \\
             & ($\boldsymbol{*}$)NARGIS-DefTuned                           & 0.718	&\textit{0.049}	&0.71	&0.148	&0.71	&0.148	&0.8	&0.105	&0.83	&0.066	&0.97	&0.017                                     \\
             & (\dag) NARGIS-PredTuned                          & 0.578                                        & 0.189                                      & 0.71                                          & \textbf{0.151}                                       & 0.71                                         & \textbf{0.151}                                      & 0.79                                      & 0.111                                    & 0.83                                      & 0.064                                    & 0.96                                      & 0.032                                    \\\dottedline{1}{14} \\ 
             & Basic                                     & 0.785                                        & -                                          & 0.87                                          & -                                           & 0.87                                         & -                                          & 0.89                                      & -                                        & 0.86                                      & -                                        & 0.93                                      & -                                        \\
             & EdgeRand$(\epsilon=6)$   & 0.407                                        & 0.378                                      & 0.61                                          & 0.26                                        & 0.61                                         & 0.26                                       & 0.68                                      & 0.206                                    & 0.77                                      & \textbf{0.091}                                    & 0.5                                       & \textbf{0.432}                                    \\
             & ($\boldsymbol{*}$)LapGraph$(\epsilon=0.1)$ & 0.429                                        & 0.356                                      & 0.51                                          & \textbf{0.36}                                        & 0.51                                         &\textbf{0.36}                                      & 0.52                                      & \textbf{0.372}                                    & 0.78                                      & 0.087                                    & 0.5                                       & \textbf{0.432}                                    \\
GCN             & (\dag) LapGraph$(\epsilon=6)$   & 0.458                                        & 0.327                                      & 0.56                                          & 0.311                                       & 0.56                                         & 0.311                                      & 0.58                                      & 0.312                                    & 0.78                                      & 0.084                                    & 0.53                                      & 0.399                                    \\
             & NARGIS                                    & 0.672                                        & \textbf{0.113}                                      & 0.8                                           & 0.068                                       & 0.8                                          & 0.068                                      & 0.85                                      & 0.039                                    & 0.86                                      & 0.007                                    & 0.85                                      & 0.081                                    \\
             & ($\boldsymbol{*}$)NARGIS-DefTuned                           & 0.545                                        & 0.24                                       & 0.57                                          & 0.296                                       & 0.57                                         & 0.296                                      & 0.84                                      & 0.051                                    & 0.88                                      & -0.011                                   & 0.82                                      & 0.107                                    \\
             & (\dag) NARGIS-PredTuned                          & 0.58                                         & \textit{0.205}                                      & 0.55                                          & 0.317                                       & 0.55                                         & 0.317                                      & 0.83                                      & 0.061                                    & 0.88                                      & -0.015                                   & 0.8                                       & 0.129                                    \\\dottedline{1}{14} \\ 
             & Basic                                     & 0.769                                        & -                                          & 0.85                                          & -                                           & 0.85                                         & -                                          & 0.89                                      & -                                        & 0.89                                      & -                                        & 0.9                                       & -                                        \\
             & ($\boldsymbol{*}$)EdgeRand$(\epsilon=6)$   & 0.413                                        & 0.356                                      & 0.51                                          & \textbf{0.347}                                       & 0.51                                         & \textbf{0.347}                                      & 0.5                                       & \textbf{0.389}                                    & 0.78                                      & \textbf{0.107}                                    & 0.5                                       & \textbf{0.398}                                    \\
             & LapGraph$(\epsilon=0.1)$ & 0.408                                        & 0.361                                      & 0.52                                          & 0.335                                       & 0.52                                         & 0.335                                      & 0.53                                      & 0.36                                     & 0.78                                      & 0.11                                     & 0.5                                       & 0.397                                    \\
GAT             & (\dag) LapGraph$(\epsilon=6)$   & 0.452                                        & 0.317                                      & 0.57                                          & 0.283                                       & 0.57                                         & 0.283                                      & 0.58                                      & 0.308                                    & 0.79                                      & 0.103                                    & 0.53                                      & 0.366                                    \\
             & NARGIS                                    & 0.657                                        & \textbf{0.112}                                      & 0.8                                           & 0.052                                       & 0.8                                          & 0.052                                      & 0.86                                      & 0.032                                    & 0.86                                      & 0.026                                    & 0.82                                      & 0.08                                     \\
             & ($\boldsymbol{*}$) NARGIS-DefTuned                           & 0.65                                         & \textit{0.119}                                      & 0.68                                          & 0.17                                        & 0.68                                         & 0.17                                       & 0.74                                      & 0.151                                    & 0.83                                      & 0.06                                     & 0.72                                      & 0.174                                    \\
             & (\dag) NARGIS-PredTuned                          & 0.641                                        & 0.128                                      & 0.68                                          & 0.178                                       & 0.68                                         & 0.178                                      & 0.74                                      & 0.15                                     & 0.84                                      & 0.054                                    & 0.71                                      & 0.19              \\\bigline                      
\end{tabular}
\end{table}
\begin{table*}[]
\centering
\caption{Defence Performances on all transfer Attacks (Attack-1, 4, 5, 7). DP-based Defense and \mymodel{} Models marked with $\dag$ and $\boldsymbol{*}$ are tuned, respectively, for high utility (Higher Node Prediction Accuracy) and high privacy (lower Attack-0 AUC).}\label{tab:transfer}
\fixtableonecol{6}
\begin{tabular}{cccccccccc|ccccccccc}\bigline
\multicolumn{2}{c}{\textbf{Dataset$\xrightarrow{}$Shadow Dataset}} & \multicolumn{8}{c|}{\textbf{Cora $\xrightarrow{}$ Citeseer}}                                                                                                                                                                                                                                                                            & \textbf{} & \multicolumn{8}{c}{\textbf{Cora $\xrightarrow{}$ Pubmed}}                                                                                                                                                                                                                                                                              \\\bigline
\textbf{GNN}                & \textbf{Defense}                                        & \multicolumn{2}{c}{\textbf{Attack-1}}                                                 & \multicolumn{2}{c}{\textbf{Attack-4}}                                                 & \multicolumn{2}{c}{\textbf{Attack-5}}                                                 & \multicolumn{2}{c|}{\textbf{Attack-7}}                                                 & \textbf{} & \multicolumn{2}{c}{\textbf{Attack-1}}                                                 & \multicolumn{2}{c}{\textbf{Attack-4}}                                                 & \multicolumn{2}{c}{\textbf{Attack-5}}                                                 & \multicolumn{2}{c}{\textbf{Attack-7}}                                                 \\
\end{tabular}
\end{table*}
\section{Discussion}\label{sec:disc}
In this work, we have introduced a Graph Modification attack as a defensive measure against Link Stealing Attack, In the bigger picture, we investigated whether a form of attack (i.e., Modification) can be introduced as a defensive measure for another form of attack (i.e., inference). Unlike the previous DP-based approaches \cite{wu2022linkteller}, where the privacy of the graph was ensured through noise addition in the adjacency matrix, i.e., modifying the edge set, our work modified (addition) in the node-set through a clustering and learning based approach. The approach was focused on keeping up the fidelity even with higher privacy, unlike DP-based approaches where there are trade-offs. From the performance analysis on different combinations of GNN models, datasets, attack methods, DP-based Defenses, \mymodel{} and its variants' performances, we could figure out some key findings:
\begin{itemize}
    \item \mymodel{} and its variants have optimal performances mostly for SAGE, and never for GCN.
    \item \mymodel{} and its variants, despite being the second best in terms of defending in most cases, are actually the most balanced form of fidelity-privacy trade-off, as they show better accuracy performances for near-equal defensive performances.
    \item \mymodel{} can preserve model fidelity better than DP-based defenses for almost similar range of defense performance.
\end{itemize}

The limitations are: 
\begin{itemize}
    \item Not having optimal performances ubiquitous for all GNN models
    \item Approaches are justified through the lens of experiments and heuristics rather than theoretical proof
    \item Unlike for similarity-based attacks, the performances on mutual influence-based LinkTeller Attack are not adequately defended, as the model considers the posterior similarity measures only- not the mutual influence between node features. 
\end{itemize}

Based on the findings, we propose some investigation scenarios and hypotheses:\\
\noindent \textbf{Formalizing the Influence of $\mathcal{\lambda}$ Weights in Defense Loss.} While tuning the model for best performances, the main adjustment was done for the loss weights $\mathcal{\lambda}_{misc}$, $ \mathcal{\lambda}_{align}$, $\mathcal{\lambda}_{dist}$ and $ \mathcal{\lambda}_{corr}$. Their corresponding loss gradients for the output layer have closed-form which could open the door for the theoretical analysis of what should be the tuning mechanism of these loss weights to achieve optimum performance. As the corresponding losses involve contrasting objectives, proper tuning scheme can help achieve the desired fidelity-privacy tradeoff.

\noindent \textbf{$\mathcal{\lambda}_{align}$ for Tuning} Our heuristic for using alignment loss weight as the tuning parameter for comparison against \dphpp and \dphfp was that this loss tries to keep the model fidelity performance up when the other losses are focusing on perturbing the posterior spaces. Jointly tuning one of the other contrasting loss weights along with $\mathcal{\lambda}_{align}$ could help as a trade-off tuning scheme. 

\noindent \textbf{Integrating Node Influence.} \mymodel{} is focused more on perturbing the model output space through the augmented nodes as the center of the pre-calculated spectral clusters, rather than through the most influential nodes. Apart from the graph topology which was used in computing the clusters, we can also integrate the nodes' influence over others while forming the clusters.

\noindent \textbf{Message-Passing Mechanism.} Our model is message-passing agnostic, as the main defensive layer is established on the graph input side, rather than in the graph message-passing mechanism, unlike \cite{zhang2020gnnguard}. As we saw SAGE's superior performance over GCN or GAT, it can be a future direction to investigate how to integrate the augmentation scheme inside the message-passing mechanism.
\section{Related Works}
In this section, we will introduce the contemporary works on attacks on graphs and GNNs, and the defensive measures against the attacks. For an in-depth introduction to the message-passing GNNs, we refer to the works \cite{kipf2016semi, kipf2016variational, veličković2018graph, hamilton2017inductive, xu2018powerful} to the readers.



\subsection{Attacks on Graph and GNN Models} Literature have discussed attacks on graph in different settings: poison (manipulating the model through corrupted data), evasion (manipulating the data to bypass the model without corrupting), reconstruction (inferring the data through model posteriors and other information), inversion (inferring the model parameters) etc. Under the evasion settings, Zou et al.~\cite{zou2021tdgia} has proposed a topological edge selection-based strategy for inserting fake nodes inside a graph, where the features are learned using a surrogate model. Chen et al. \cite{chen2022understanding} have shown that by improving the homophily unnoticeability of the graphs, Graph Injection attacks, another example of evasion setting, can outperform previous homophily-based attacks. Different attacks have been proposed for inferring the graph structures and memberships. He et al.~\cite{he2021stealing} introduced Link stealing attack from the posteriors of GNNs on the basis of proximity-based features. In \cite{wu2022linkteller}, Wu et al., have evolved from the notion of similarity between node posteriors and leverage node feature perturbation-based influence analysis for predicting graph edges. Zhang et al. ~\cite{zhang2023demystifying} has shown that the privacy risks for edge groups from the membership inference attack are uneven and performed group-based attacks depending on the unequal vulnerability. Different works have also discussed poisoning the graph for reducing the node classification performance through Meta-Learning (Meta-attack \cite{zugner2020adversarial}), Reinforcement Learning (NIPA \cite{sun2020adversarial}), Fast Gradient Sign
Model linearization (AFGSM \cite{wang2020scalable}) etc. It should be noted that our focus attack is in reconstruction settings. In contrast, our defense is influenced by poisoning settings, as it manipulates the graph structure during training. Defending the model and data from an attack under a particular setting through adapting another attack settings is the main focus of our work. 

\subsection{Defense against Attack on Graph Models} Several defense mechanisms have been proposed and validated for attack on Graphs and GNNs. The defensive strategies for these works are mostly focused on robust perturbations of Graph embeddings and Differential Privacy guaranteed learning of Graphs. Zhang and Zitnik has introduced
GNNGuard~\cite{zhang2020gnnguard}- based on neighbor importance estimation and layer-wise graph memory in order to prune possible fake and suspicious edges to stabilize the original graph against the training-time perturbation attack. Differential privacy-based approaches have been a popular method for defending against graph reconstruction (specially link stealing) attacks. Kolluri et al. ~\cite{kolluri2022lpgnet} developed a novel architecture named LPGNet using only multilayer perceptrons to model both the node feature information and some carefully chosen graph structural information from the graph edges after compressing the edge information as a feature vector. In this way, edges are kept separate from the attacks. Another approach called Blink (\cite{zhu2023blink}) injects noise into each node’s adjacency list and degree in a decentralized setting and guarantees Localized Differential Privacy; and then uses Bayesian estimation in the server to receive original edge information. Zhou et al. have introduced the Markov chain-based graph reconstruction attack and defense under information theory-guided assumptions \cite{zhou2023strengthening}. Wu et al., apart from LinkTeller, also introduced in \cite{wu2022linkteller} two DP-based defenses named EdgeRand and LapGraph against link inference attacks, as discussed in previous sections. Our work differs from these DP-based reconstruction-defending approaches in different aspects. Firstly, unlike the DP-based approaches, our model does not change the original edges and node features through noise injection. Hence, the originality of the concerned data is preserved. Secondly, we leverage a deterministic approach- spectral clustering, for injecting the nodes in focused places- instead of the randomized approach prioritized by the DP-based models. This can pave the way for more stable and consistent performance analysis and guarantees.



\section{Conclusion and Future Work}

In this paper, we introduce a defense against link-stealing attacks inspired by the paradigm of using attacks as a form of defense. We propose \model{} - a node augmentation-based defense for GNN, which can restrict the model from providing posteriors leading to edge inference attackers, by perturbing the embedding space through augmenting nodes with learnable features instead of fixed ones. We show that our model can achieve good fidelity-privacy tradeoff in some cases, while there are other scopes to improve on.

\noindent \textbf{Future work.} In the future we want to theoretically find a performance bound - concerning the tunable defense loss weights in the model optimization loop. We also would like to investigate how to make the model performance better across different message-passing mechanisms. 
\bibliographystyle{ACM-Reference-Format}
\bibliography{main}

\clearpage
\appendix


\end{document}